\documentclass{article}
\PassOptionsToPackage{numbers,compress}{natbib}
\usepackage[preprint]{neurips_2025}

\usepackage{amsmath,amssymb,amsthm,graphicx}
\usepackage{microtype}   
\usepackage{xurl}        
\usepackage{array}       
\usepackage{subfigure}
\usepackage{hyperref}
\theoremstyle{plain}
\newtheorem{theorem}{Theorem}[section]
\newtheorem{lemma}[theorem]{Lemma}
\newtheorem{proposition}[theorem]{Proposition}

\theoremstyle{definition}
\newtheorem{definition}[theorem]{Definition}
\newtheorem{assumption}{Assumption}
\theoremstyle{remark}
\newtheorem{remark}[theorem]{Remark}
\newenvironment{keywords}{\par\smallskip\noindent\textbf{Keywords: }}{\par\smallskip}
\DeclareMathOperator*{\E}{\mathbb{E}}
\DeclareMathOperator*{\Prob}{\mathbb{P}}

\DeclareMathOperator{\softmax}{softmax}
\newcommand{\R}{\mathbb{R}}

\newcommand{\ind}{\mathbb{I}}

\newcommand{\JS}{D_{\text{JS}}}
\newcommand{\KL}{D_{\text{KL}}}
\renewcommand{\ln}{\log}

\makeatletter
\@ifundefined{localnatbibstub}{}{%
\newcommand{\aycite}[1]{%
  \@ifundefined{b@#1}{#1}{%
    \edef\aycite@label{\csname b@#1\endcsname}%
    \expandafter\aycite@helper\aycite@label\@nil
  }%
}
\def\aycite@helper#1,#2\@nil{#1 (\ignorespaces#2)}
\renewcommand{\cite}[2][]{%
  \begingroup
  \def\aycite@first{1}%
  \@for\aycite@key:=#2\do{%
    \ifnum\aycite@first=1\else;\space\fi
    \aycite{\aycite@key}%
    \def\aycite@first{0}%
  }%
  \endgroup
}
}
\makeatother

\title{On the Trade-off Between Information Loss and Generalization in Sparse Attention}

\author{%
  Zhongqi Fan\thanks{Equal contribution.} \\
  Beijing Normal-Hong Kong Baptist University, Zhuhai, China \\
  \texttt{u430033011@mail.bnbu.edu.cn} \\
  \And
  Zheng Tan\footnotemark[1] \\
  Xi'an Jiaotong University, Xi'an, Shaanxi, China \\
  \texttt{zhengtan@stu.xjtu.edu.cn} \\
}

\begin{document}

\maketitle

\begin{abstract}
Sparse attention is commonly presented as an engineering compromise: it reduces the quadratic cost of self-attention at the price of information loss. This paper instead studies sparse attention as a structural inductive bias and develops a unified approximation--generalization view of its cost and benefit. We first characterize the approximation error between dense and sparse attention through the Jensen--Shannon (JS) divergence. Under an exact weight-scaling relation, the JS divergence admits a one-dimensional closed form determined only by the truncation mass $\alpha$; an order-statistics concentration analysis then shows that $\alpha$ grows with the sequence length under a mean-field, independent-logit baseline. We next quantify the resulting capacity reduction through Rademacher complexity. For fixed-support sparse attention, which serves as a tractable baseline before dynamic Top-$k$ routing is considered, Dudley's integral and a volumetric covering argument yield a bound of order $\mathcal{O}(\gamma\sqrt{M/n}(\sqrt{\log(3eL/M)}+\sqrt{\pi}/2))$, together with a layer-wise Lipschitz constant whose $H$-layer cascade exposes the stabilizing role of normalization and residual connections. Because a purely geometric bound conflates discrete support selection with continuous numerical precision, we further use a Xu--Raginsky mutual-information decomposition to separate the support entropy from the quantization resolution $\epsilon$, obtaining the bound $\mathcal{O}(\sqrt{(M/2n)(\log(eL/M)+\log(1+2/\epsilon))})$. These results yield a common trade-off: increasing the sparse budget $M$ enlarges the capacity term, while decreasing $M$ increases the approximation error measured by the JS divergence. We translate the resulting bounds into a long-sequence budget rule $M(L)=\mathcal{O}(\sqrt{L})$, a multi-head/narrow-head preference, and a sparse--quantization co-design principle. Experiments with GPT-2 attention confirm the closed-form JS prediction and the predicted direction of information decay, while the theoretical bound remains conservative.
\end{abstract}


\section{Introduction}
\label{sec:introduction}

The Transformer architecture~\citep{vaswani2017attention} is a standard backbone for sequence modeling, but the quadratic cost of self-attention limits long-context processing. Sparse attention reduces this cost by restricting each query to a small support, and fixed, local, and content-based variants have achieved strong empirical results~\citep{wang2020linformer,beltagy2020longformer,roy2021efficient}. Two theoretical questions remain: how much information is lost by replacing dense attention with a sparse distribution, and how does this approximation cost interact with the statistical capacity of the model? We study fixed-support attention first because it isolates the structural capacity effect; dynamic Top-$k$ routing is revisited as an extension.

Our central perspective is that sparsification is a structural inductive bias, not only an efficiency compromise. The cost is the truncation of probability mass; the benefit is a reduction in capacity and variance. Three quantities make this trade-off explicit. The JS divergence measures the approximation cost and reduces it to the truncation mass $\alpha$. Rademacher complexity measures the geometric capacity of the fixed-support hypothesis class and exposes its dependence on the sparse budget $M$ and depth $H$. Since a geometric bound alone does not separate discrete support selection from continuous numerical precision, a Xu--Raginsky mutual-information decomposition provides the missing separation between support entropy and the resolution $\epsilon$ of the sparse weights. This roadmap does not introduce an additional modeling assumption. It simply states that the approximation cost is measured by the JS divergence through the truncation mass, the geometric capacity is measured by the Rademacher bound, and the topology--precision trade-off is measured by the mutual-information bound. The main contributions are as follows.
\begin{itemize}
\item \textbf{Approximation metric.} We derive a closed-form JS divergence determined solely by $\alpha$ and a high-probability long-sequence analysis of $\alpha$ under a mean-field, independent-logit baseline.
\item \textbf{Geometric capacity.} We bound the Rademacher complexity of fixed-support sparse attention and quantify the layer-wise Lipschitz cascade behind normalization and residual stabilizers.
\item \textbf{Information-theoretic decoupling.} We decompose the entropy of the sparse hypothesis into a support-selection term $M\log(eL/M)$ and a precision term $M\log(1+2/\epsilon)$.
\item \textbf{Design guidance.} We convert the bounds into a long-sequence budget rule, a multi-head, narrow-head principle, and a sparse--quantization co-design rule, and we verify the JS prediction on GPT-2 attention.
\end{itemize}
Section~\ref{sec:related} reviews related work; Section~\ref{sec:core} fixes notation and develops the information-decay analysis; Section~\ref{sec:generalization} presents the two generalization views; Section~\ref{sec:implications} gives design guidance; and Section~\ref{sec:conclusion} concludes. Proofs and experimental details are deferred to the appendices.

\section{Related Work}
\label{sec:related}

This section reviews the sparse-attention literature, the theoretical analysis of transformers most relevant to our goals, and the tools we build on.

\subsection{Efficient Transformer Architectures}
Sparse attention methods replace the dense attention matrix with a smaller set of connections. Sparse Transformer introduced factorized patterns, Longformer and BigBird combine local windows with global or random tokens, Reformer uses locality-sensitive hashing, and more recent architectures such as NSA are hardware-aligned and trainable~\citep{child2019generating,beltagy2020longformer,zaheer2021bigbirdtransformerslonger,kitaev2020reformer,yuan2025native}. These methods demonstrate that sparsity can preserve long-range modeling capacity, but they do not quantify the information loss or its statistical-learning consequences.

\subsection{Expressivity and Generalization in Transformers}
Theoretical investigations of Transformer architectures have largely developed along two complementary lines: representation capacity and statistical generalization.
On the representation front, foundational results establish universal approximability for dense as well as sparse Transformers~\citep{yun2020n}, and analyze the expressive benefits of random projections~\citep{likhosherstov2021expressive}. More recently, circuit and communication complexities have emerged as standard lenses for probing the intrinsic limitations of multi-layer attention~\citep{chen2024theoreticallimitationsmultilayertransformer}, revealing, for instance, that standard self-attention provably excels at sparse pairwise aggregation yet faces fundamental hardness on higher-order relational tasks~\citep{sanford2023representational}.

On the statistical front, existing risk analyses typically bound the sample complexity of Transformers via sequence-length-independent norm bounds~\citep{trauger2023sequence}, channel decoding formulations~\citep{zhang2026generalization}, or covering-number arguments under low-dimensional manifold data structures~\citep{havrilla2024understanding}. However, these statistical accounts almost exclusively characterize dense architectures, treating connection pruning as an ad-hoc implementation artifact. A unified learning-theoretic framework that couples the approximation cost of structural sparsification with the statistical capacity contraction of the resulting hypothesis class remains missing.

\subsection{Information-Theoretic and Geometric Generalization Frameworks}
To characterize this dual approximation--generalization behavior, we synthesize three theoretical instruments operating at distinct granularities:
\begin{itemize}
    \item \textbf{Distributional Approximation via JS Divergence.} Unlike the Kullback--Leibler (KL) divergence, which suffers from severe support-mismatch singularities when dense attention masses are strictly zeroed out, the Jensen--Shannon (JS) divergence is symmetric and naturally bounded~\citep{lin1991divergence}. It provides a well-conditioned discrepancy metric and is intrinsically linked to representation learning guarantees~\citep{dorent2025connectingjensenshannonkullbackleiblerdivergences,hjelm2019learningdeeprepresentationsmutual}.
    \item \textbf{Geometric Capacity via Rademacher Complexity.} Empirical Rademacher complexity offers a distribution-dependent capacity measure~\citep{bartlett2002rademacher}. Governed by Dudley's entropy integral and contraction principles, it captures the multi-layer Lipschitz propagation and layer-wise stabilization effects across transformer blocks.
    \item \textbf{Hypothesis Discretization and Mutual-Information Bounds.} The classical framework by \citet{xu2017information} connects expected generalization error to the input--output mutual information $I(S; W)$. While mutual information in overparameterized networks can easily become vacuous or intractable, recent advances demonstrate that parameter subspace slicing, rate-distortion regularization, and quantization effectively compress the functional hypothesis space~\citep{nadjahi2024slicing}. Moreover, \citet{chu2023unified} show that generic chaining and information-theoretic bounds can be reconciled within Orlicz-space decorelation techniques.
\end{itemize}
Deploying these tools in tandem allows us to isolate the approximation degradation ($D_{\mathrm{JS}}$), the geometric capacity reduction ($\mathcal{R}_n$), and the decoupled topology within a coherent statistical view.

\section{Information Decay: Closed-Form Divergence and Concentration}
\label{sec:core}

This section fixes the notation and quantifies the information cost of truncation. We first derive the closed-form divergence as a function of the truncation mass, then characterize the tail behavior of the logits that drives it, and finally study how the truncation mass concentrates as the sequence length grows.

\subsection{Notation and Problem Formulation}
We consider a sequence of length $L$ and write $[L]=\{1,\dots,L\}$. For a fixed query position $i$, the dense attention distribution is a probability vector on the simplex, while a sparse attention mechanism retains only a support set $\mathcal{S}_i\subseteq[L]$ of size $M=|\mathcal{S}_i|$. The difference between the two distributions is summarized by the discarded probability mass, or truncation mass, $\alpha_i$. Definitions~\ref{def:attention_distributions} and~\ref{def:sparse_quality} fix these objects and are used throughout the paper.

\begin{definition}[Full and Sparse Attention Distributions]
For the $i$-th query, the full attention distribution is the softmax of the attention score vector $S_i = [S_{i1}, \dots, S_{iL}]$:
\[
p_i = \softmax(S_i) \in \Delta^{L-1}, \quad p_i(j) = \frac{\exp(S_{ij})}{\sum_{k=1}^L \exp(S_{ik})}, \forall j \in [L]
\]
where $\Delta^{L-1}$ denotes the $L$-dimensional probability simplex.
The support set $\mathcal{S}_i \subseteq [L]$ of the sparse mechanism is the set of indices to which the $i$-th query attends.
The sparse attention score follows distribution:
\[
q_i = \softmax\left( \tilde{S}_i \right) \in \Delta^{L-1}, \quad \tilde{S}_i(j) = \begin{cases} S_{ij} & j \in \mathcal{S}_i \\ -\infty & j \notin \mathcal{S}_i \end{cases}
\]
By construction, $q_i(j) = 0$ for all $j \notin \mathcal{S}_i$.
\label{def:attention_distributions}
\end{definition}

To quantify the approximation error between $p_i$ and $q_i$, we define the truncation mass of the sparse mechanism:

\begin{definition}[Truncation Mass]
The truncation mass $\alpha_i$ is the total probability mass of the full attention distribution discarded by the sparse mechanism:
\[
\alpha_i = \sum_{j \notin \mathcal{S}_i} p_i(j)
\]
By definition, $\alpha_i \in [0,1]$, with $\alpha_i=0$ corresponding to no truncation and $\alpha_i=1$ corresponding to complete truncation of the attention mass.
\label{def:sparse_quality}
\end{definition}

\begin{remark}[Fixed-support scope]
\label{rem:fixed-support}
The fixed-support assumption used in the first part of the paper means that $\mathcal{S}_i$ is chosen independently of the input, whereas practical routing mechanisms such as Top-$k$ attention choose the support adaptively. We study the fixed-support case first because it isolates the capacity reduction caused by sparsity itself and avoids the non-Lipschitz discontinuity introduced by a data-dependent selection rule. The resulting bounds should therefore be read as a capacity baseline for data-dependent sparse routing, not as a claim about every adaptive mechanism; Section~\ref{sec:conclusion} returns to this distinction.
\end{remark}

\subsection{Closed-Form Jensen--Shannon Divergence}

\begin{definition}[Jensen--Shannon Divergence]
For two probability distributions $p$ and $q$ over the same support, the JS divergence is defined as:
\[
\JS(p \| q) = \frac{1}{2} \KL(p \| m) + \frac{1}{2} \KL(q \| m)
\]
where $m = \frac{1}{2}(p + q)$ is the average distribution of $p$ and $q$, and $\KL(p \| m) = \sum_{j} p(j) \ln \frac{p(j)}{m(j)}$ is the Kullback-Leibler (KL) divergence.
\end{definition}

The JS divergence is a natural approximation metric here because it is symmetric and bounded on distributions with non-overlapping support. In contrast, the KL divergence between a dense distribution and a sparse distribution can be infinite when the sparse support misses positive-mass events. The following weight-scaling relation shows that the dense and sparse weights are aligned on the retained support; this is the structural fact that makes the JS divergence one-dimensional.

\begin{lemma}[Weight Scaling Lemma]
\label{lemma:weight_scaling}
For all $j \in \mathcal{S}_i$, the full attention weight and sparse attention weight satisfy the exact relationship:
\[
p_i(j) = (1 - \alpha_i) q_i(j)
\]
\end{lemma}

We now state the exact reduction from the probability vectors to the scalar $\alpha$.

\begin{theorem}
The Jensen--Shannon divergence between the full attention distribution $p_i$ and sparse attention distribution $q_i$ is a function \emph{only of the truncation mass $\alpha = \alpha_i$}, with the closed-form expression:
\[
\JS(p_i \| q_i) = \ln 2 + \frac{1-\alpha}{2} \ln(1-\alpha) - \frac{2-\alpha}{2} \ln(2-\alpha), \quad \alpha \in [0,1]
\]
Moreover, this function is strictly convex and strictly increasing on $\alpha \in (0,1)$.
\label{thm:js_divergence}
\end{theorem}

Theorem~\ref{thm:js_divergence} shows that two seemingly high-dimensional distributions have a one-dimensional approximation cost: once the retained and discarded masses are fixed, the pointwise arrangement inside the retained support cannot change the JS divergence. The function $f(\alpha)$ is strictly increasing and strictly convex, so additional truncation always increases the approximation error, and the loss grows at an increasing marginal rate. The proof is given in Appendix~\ref{app:A}.

\subsection{Sub-Gaussianity of Attention Logits}

The concentration analysis requires a tail condition on the scores. We first state the two norm assumptions under which the logits are sub-Gaussian, recall the sub-Gaussian definition below, and record the consequence of bounded embedding and projection spectral norms.

\begin{assumption}[Bounded projection spectral norms]
\label{ass:bounded_spectral_norm}
There exists a constant $C_W$ such that
\[
\|W_Q\|_2 \leq C_W \quad \text{and} \quad \|W_K\|_2 \leq C_W .
\]
\end{assumption}

\begin{assumption}[Bounded embedding norms]
\label{ass:bounded_embedding_norm}
Let $X_i \in \mathbb{R}^d$ be the input token embeddings, which satisfy $\|X_i\|_2 \leq \gamma$.
\end{assumption}

\begin{definition}[Sub-Gaussianity, \citealp{vershynin2020high}]
A random variable $X$ is sub-Gaussian if its tail distribution satisfies $\Prob(|X| \geq t) \leq 2 \exp(-ct^2/\sigma^2)$. This is equivalent to requiring its sub-Gaussian norm to be finite, that is, $\|X\|_{\psi_2} = \inf \{ t > 0 : \mathbb{E}[\exp(X^2/t^2)] \leq 2 \} < \infty$. For a zero-mean sub-Gaussian variable, its moment generating function satisfies $\mathbb{E}[\exp(\lambda X)] \leq \exp(C\lambda^2\|X\|_{\psi_2}^2)$ for all $\lambda \in \mathbb{R}$.
\end{definition}

\begin{proposition}[Sub-Gaussianity of Attention Scores]
Let $X_i, X_j \in \R^d$ be input token embeddings with the upper bound $\gamma$, and $W^Q, W^K \in \R^{d \times d_{qk}}$ be the query and key projection matrices with spectral norm bounds $\|W^Q\|_2 \leq C_W$, $\|W^K\|_2 \leq C_W$. The scaled attention score is defined as:
\[
A_{ij} = \frac{1}{\sqrt{d_{qk}}} \left(X_i W^Q\right)^\top \left(X_j W^K\right)
\]
Then $A_{ij}$ is a sub-Gaussian random variable with $\|A_{ij}\|_{\psi_2}\leq\frac{\gamma^2 C_W^2}{\sqrt{\ln 2 \cdot d_{qk}}}$.
\label{prop:subgaussian_attention}
\end{proposition}

The proposition makes the role of the head dimension explicit: $\|A_{ij}\|_{\psi_2}\le \gamma^2 C_W^2/\sqrt{\ln 2 \cdot d_{qk}}$, so $d_{qk}$ enters the tail parameter through the factor $d_{qk}^{-1/2}$. At this point, $d_{qk}$ is a statistical parameter of the logits; Section~\ref{sec:implications} will convert it into an architectural recommendation.

\subsection{Long-Sequence Concentration of the Truncation Mass}

The sub-Gaussianity of the attention scores lays the probabilistic foundation for the order-statistics-based concentration analysis. We now define the partition functions for the in-cluster and out-of-cluster attention terms:
\[
Z_{\text{in}} = \sum_{j \in \mathcal{S}_i} \exp(S_{ij}), \quad Z_{\text{out}} = \sum_{j \notin \mathcal{S}_i} \exp(S_{ij})
\]

By definition, the truncation mass can be rewritten as:
\[
\alpha_i = \frac{Z_{\text{out}}}{Z_{\text{in}} + Z_{\text{out}}} = \frac{1}{\frac{Z_{\text{in}}}{Z_{\text{out}}} + 1}
\]

This identity shows that bounding the truncation mass is equivalent to bounding the ratio $R = \frac{Z_{\text{in}}}{Z_{\text{out}}}$: since $\alpha_i = 1/(R+1)$ is strictly decreasing in $R$, an upper bound on $R$ yields a lower bound on $\alpha_i$.

\begin{assumption}[Independent and identically distributed attention scores]
\label{ass:independent_attention_scores}
For a fixed query position $i$, the attention scores $S_{i1}, S_{i2}, \dots, S_{iL}$ are independent and identically distributed sub-Gaussian random variables with sub-Gaussian norm at most $\sigma$, where $\sigma$ does not depend on the sequence length $L$.
\end{assumption}

\emph{Centering.} All tail bounds in Appendix~\ref{app:prob_tools} are stated for the centered scores $\tilde{S}_{ij} := S_{ij} - \mathbb{E}[S_{ij}]$. Under the i.i.d.\ assumption this centering subtracts the same constant from every score of a row, that is, it is a common shift of the attention logits; since the softmax is invariant under a common shift, the attention distributions $p_i$ and $q_i$---and hence the truncation mass $\alpha_i$ and the ratio $R$---are left unchanged. The centered scores remain sub-Gaussian with the same norm, $\|\tilde{S}_{ij}\|_{\psi_2} = \|S_{ij}\|_{\psi_2} \leq \sigma$.

Our final core result provides a high-probability bound for this ratio, and characterizes its scaling with sequence length:

\begin{theorem}
Assume that the attention scores of the $i$-th query, $S_{i1}, \dots, S_{iL}$, are independent and identically distributed sub-Gaussian random variables with sub-Gaussian norm at most $\sigma$ (Assumption~\ref{ass:independent_attention_scores}). Let $\delta \in (0,1)$ be the confidence level, with the following probability estimate
\[
\Prob\left( \frac{Z_{\text{in}}}{Z_{\text{out}}} \leq \frac{M}{L-M} \exp\left(2\sqrt{2}\,\sigma\sqrt{\ln(eL)}+\sqrt{2}\sigma+4 \sigma \sqrt{ \ln \frac{2}{\delta}  } \right) \right) \geq 1 - \delta
\]
where \(M=|\mathcal{S}_i|\).
\label{thm:sparse_decay}
\end{theorem}

\begin{remark}
\label{rem:dilution}
We interpret $M$ as the dilution degree. Under the commonly used budgets $\mathcal{O}(L^a)$ with $a<1$ and $\mathcal{O}(\ln L)$, the high-probability upper bound on $R$ decreases as $L$ increases, so the implied lower bound on the truncation mass $\alpha$ increases with $L$: sparse attention loses more information as the sequence lengthens. Since the derivation uses only $|\mathcal{S}_i|=M$ and never which indices are retained, the bound applies to any $M$-sparse support, with the top-$M$ set as the most favorable case among them. The numerical illustration and the support-agnostic reading of the bound are deferred to Appendix~\ref{app:concentration}.
\end{remark}

A qualification is important for interpreting Theorem~\ref{thm:sparse_decay}: the independent and identically distributed sub-Gaussian logit assumption is a mean-field baseline rather than a literal description of trained attention. It deliberately removes local clustering, attention sinks, and prompt-specific routing so that the resulting bound isolates the information cost of truncation under no structural prior. Real attention is in general neither independent nor identically distributed, so the theorem is best read as a conservative baseline: it identifies a worst-case direction of information decay rather than claiming tight numerical agreement. Extending the analysis to weakly dependent logits is a natural step and would require mixing or conditional-independence assumptions in place of Assumption~\ref{ass:independent_attention_scores}.

Two quantities control the bound in Theorem~\ref{thm:sparse_decay}: the order of the sparsity budget $M$ and the coefficient $C = 2\sqrt{2}\sigma$, where
\[
\sigma = \frac{\gamma^2 C_W^2}{\sqrt{\ln 2 \cdot d_{qk}}},
\]
so $C$ is jointly determined by the embedding norm $\gamma$, the projection spectral norm $C_W$ and the head dimension $d_{qk}$. For $M = \mathcal{O}(\sqrt{L})$ the prefactor $M/(L-M)$ decays as $L^{-1/2}$, while for $M = \mathcal{O}(\ln L)$ it decays almost as $L^{-1}$; and since a smaller $C$ gives a smaller upper bound on $R$ and hence a larger lower bound on $\alpha$, both $M$ and $C$ govern how fast the information loss accumulates.

Since $\sigma \propto d_{qk}^{-1/2}$, a smaller head dimension---equivalently, more heads at a fixed model width---slows the growth of the truncation mass $\alpha$, subject to the entropy-collapse constraint on $\gamma$ and $C_W$ \citep{zhai2023stabilizing,qi2025taming}. Appendix~\ref{app:norm_stats} reports these constants on BERT-base.

\begin{remark}
\label{rem:numeric_constants}
With the representative magnitudes $\gamma = 1$ and $C_W = 2$ measured on BERT (Appendix~\ref{app:norm_stats}), the coefficient is $C \approx 1.7$ at $d_{qk} = 64$ and $C \approx 1.2$ at $d_{qk} = 128$; since $d_{qk} = d_{\text{model}}/h$, these settings correspond to $h = 12$ and $h = 6$ heads at $d_{\text{model}} = 768$. The resulting decay of the bound, and the ranges over which these constants were measured, are reported in Appendix~\ref{app:concentration}.
\end{remark}

\subsection{Real-Model Verification of the Information Loss}
\label{sec:exp-js}

We verify the two ingredients of Section~\ref{sec:core} on GPT-2 (124M) with \texttt{WikiText-103} test texts, averaging over $12\times12\times8$ attention rows per length (12 layers, 12 heads, 8 texts).

\begin{figure}[ht]
\centering
\includegraphics[width=0.72\linewidth]{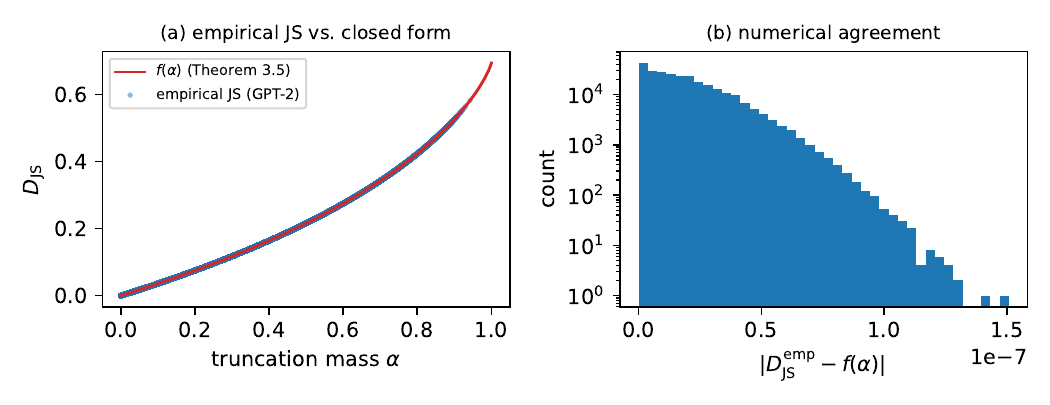}
\caption{Verification of Theorem~\ref{thm:js_divergence} on GPT-2 attention. (a) Empirical $D_{\mathrm{JS}}(p_i\|q_i)$ against the closed form $f(\alpha)$, over $2.7\times10^5$ attention rows with $\alpha\in[0,0.943]$. (b) Histogram of the absolute difference; the maximum is $1.5\times10^{-7}$.}
\label{fig:js_theorem5}
\end{figure}

\paragraph{Verification of Theorem~\ref{thm:js_divergence}.} Truncating each attention row at $L=256$ to its top-$m$ entries for ten retention ratios, the empirical divergence matches the closed form $f(\alpha)$ of Theorem~\ref{thm:js_divergence} to within $1.5\times10^{-7}$ over $2.7\times10^5$ rows with $\alpha\in[0,0.943]$ (Figure~\ref{fig:js_theorem5})---the float32 precision of the attention weights---so a truncated row is fully characterized by the single scalar $\alpha$, as claimed.

\paragraph{Information Loss Versus Sequence Length.} For six strategies with their own budgets $M(L)$---top-$M$ with $M=\lceil\sqrt{L}\rceil$ and $M=\lceil\ln L\rceil$, a window, a window plus random tokens, a dilated pattern, and uniform-random tokens---the full sweep is given in Appendix~\ref{app:strategy-loss}. The strategy dominates the sequence length: at $L=1024$ and $M=32$, top-$M$ truncation discards $\alpha=0.20$ of the attention mass, against $0.63$ for a window and $0.97$ for a random pattern, and the loss $D_{\mathrm{JS}}$ grows with $L$ for every content-agnostic strategy ($0.291\to0.345$ window, $0.210\to0.358$ dilated, $0.548\to0.647$ random); even the slowest budget tested rises from $0.069$ to $0.086$. The direction agrees with Theorem~\ref{thm:sparse_decay}, but the guaranteed bound stays one to three orders of magnitude looser, since the i.i.d.\ sub-Gaussian model does not capture how strongly real attention concentrates on a few keys.

\section{Generalization Bounds: Geometric and Information-Theoretic Views}
\label{sec:generalization}

While the preceding section quantifies the ``information cost'' of sparsification via the JS divergence, such a loss in expressive fidelity is not merely a detrimental side effect. From the perspective of statistical learning theory, the sparsity of attention connections functions as a structural constraint on the model's functional capacity. To provide a holistic evaluation of sparse attention, it is imperative to investigate how this information-theoretic penalty translates into gains in generalization stability.

We first adopt a geometric viewpoint based on Rademacher complexity. Because that viewpoint treats the sparsity pattern and the numerical weights together, we then introduce an information-theoretic bound that separates the discrete support-selection term from the precision term.

\subsection{Geometric Perspective via Rademacher Complexity}

We study the geometric capacity of a static sparse aggregation
baseline. The purpose of this analysis is to isolate the effect of
sparse connectivity from the additional complexity introduced by
input-dependent attention weights.

For a standard Transformer attention layer, the attention matrix is
given by
\[
A(X) = \operatorname{softmax}
\left(
\frac{XW_Q(XW_K)^\top}{\sqrt{d_{qk}}}
\right),
\]

where the softmax is applied row-wise. The corresponding attention
output is
$
\mathbf{f}(X) = A(X)XW_V.
$

In general, the attention coefficients depend on the input through
the query and key projections. Consequently, a covering-number
argument for a class with constant aggregation weights does not,
by itself, establish a Rademacher complexity bound for the full
input-dependent attention class.

We therefore consider a restricted static sparse aggregation
baseline. Let each input be a sequence
$x_{1:L}=(x_1,\ldots,x_L)\in\mathcal{X}^L$, where
$x_s\in\mathbb{R}^d$ and
$\|x_s\|_2\leq\gamma,
\qquad s\in[L].$

For a fixed output coordinate $p\in[d]$, define
\[
\mathcal{H}
=
\left\{
h(x_{1:L})
=
\sum_{t=1}^{L}\lambda_t
\sum_{s\in\mathcal{S}_t}
w_{t,s}[x_s]_p
\;\middle|\;
\begin{array}{l}
\lambda\in\Delta^{L-1},\\
\|w_{t,\cdot}\|_1\leq 1,\\
|\mathcal{S}_t|\leq M
\quad\text{for every }t\in[L]
\end{array}
\right\},
\]

where
$\Delta^{L-1}=\{\lambda\in\mathbb{R}^L:
\lambda_t\geq0,\ \sum_{t=1}^L\lambda_t=1\}$.
Each support $\mathcal{S}_t\subseteq[L]$ is independent of the
input and may range over the candidate supports of size at most
$M$. The coefficients $w_{t,s}$ are also input-independent.
Thus, $\mathcal{H}$ is a static sparse aggregation class rather
than the full class of input-dependent sparse attention layers.

\begin{definition}[Rademacher Complexity]
\label{def:rademacher}
Given a function class $\mathcal{G}$ mapping from
$\mathcal{X}^L$ to $\mathbb{R}$ and a sample
$S=(x^{(1)},\ldots,x^{(n)})$, its empirical Rademacher complexity is

\[
\widehat{\mathcal{R}}_S(\mathcal{G})
=
\mathbb{E}_{\boldsymbol{\sigma}}
\left[
\sup_{g\in\mathcal{G}}
\frac{1}{n}\sum_{i=1}^{n}
\sigma_i g(x^{(i)})
\right],
\]

where $\sigma_1,\ldots,\sigma_n$ are independent Rademacher
random variables taking values in $\{-1,+1\}$ with equal
probability. The expected Rademacher complexity is

\[
\mathcal{R}_n(\mathcal{G})
=
\mathbb{E}_{S\sim D^n}
\left[
\widehat{\mathcal{R}}_S(\mathcal{G})
\right].
\]

\end{definition}

The following result bounds the capacity of the static sparse
aggregation class.

\begin{theorem}[Rademacher Complexity of the Static Sparse
Aggregation Class]
\label{thm:sparse-linear-rademacher}
Suppose that $\|x_s\|_2\leq\gamma$ for every input sequence and
every position $s\in[L]$. Then

\[
\mathcal{R}_n(\mathcal{H})
\leq
C\gamma\sqrt{\frac{M}{n}}
\left(
\sqrt{\log\frac{3eL}{M}}
+
\frac{\sqrt{\pi}}{2}
\right),
\]

where $1\leq M\leq L$, $n$ is the sample size, and $C$ is an
absolute constant.
\end{theorem}

The proof uses Dudley's entropy integral
\citep{dudley1967sizes} and a volumetric covering-number argument
\citep{vershynin2020high}. The resulting bound makes explicit how
the capacity of the static sparse aggregation class depends on
the sparsity budget $M$, the sequence length $L$, and the sample
size $n$. It does not directly bound the Rademacher complexity
of the full input-dependent attention class.

We next derive a generalization bound for the same static class.

\begin{theorem}[Generalization Bound for Static Sparse
Aggregation]
\label{thm:generalization-sparse-linear}
Let $\mathcal{H}$ be the static sparse aggregation class defined
above. Let
$\ell:\mathbb{R}\times\mathcal{Y}\to[0,1]$ be a loss function
that is $1$-Lipschitz in its first argument. Suppose
$(X_1,Y_1),\ldots,(X_n,Y_n)$ are i.i.d.\ samples drawn from a
distribution $P$ on $\mathcal{X}^L\times\mathcal{Y}$.
Then, for any $\delta\in(0,1)$, with probability at least
$1-\delta$ over the draw of the sample, every $h\in\mathcal{H}$
satisfies
\[
\begin{aligned}
\mathbb{E}\big[\ell(h(X),Y)\big]
\leq{}&
\frac{1}{n}\sum_{i=1}^{n}
\ell(h(X_i),Y_i)
+
C\gamma\sqrt{\frac{M}{n}}
\left(
\sqrt{\log\frac{3eL}{M}}
+
\frac{\sqrt{\pi}}{2}
\right)
+
\sqrt{\frac{8\ln(2/\delta)}{n}},
\end{aligned}
\]
where $C$ is an absolute constant, which may differ from the
constant in Theorem~\ref{thm:sparse-linear-rademacher}.
\end{theorem}

Theorem~\ref{thm:generalization-sparse-linear} follows from the
Rademacher complexity bound and a standard uniform
generalization inequality
\citep{bartlett2002rademacher}. Both results concern the static
sparse aggregation class $\mathcal{H}$. Extending these guarantees
to attention layers whose aggregation weights depend on the input
requires a separate analysis of the corresponding function class.

Beyond the static aggregation baseline, the Lipschitz continuity of a fixed-support sparse attention layer $f(x)$ (with respect to the input tokens) characterizes how perturbations propagate across $H$ stacked layers.

\begin{lemma}[Lipschitz continuity of sparse attention]
\label{lem:lipschitz-attention}
Let $\mathbf{f}(x)$ be a sparse attention layer with parameters $(W_Q, W_K, W_V)$ and a fixed support set $\mathcal{S}$ with $|\mathcal{S}| = M$ that does not depend on the input. Assume that
\[
\|x_t\|_2 \le \gamma \ \text{ for all positions } t, \qquad
\|W_Q\|_2, \|W_K\|_2, \|W_V\|_2 \le C_W .
\]
Then $\mathbf{f}$ is Lipschitz continuous under the $\ell_\infty$ norm, with
\[
\|\mathbf{f}(x) - \mathbf{f}(y)\|_\infty
\le L_{\text{att}} \|x - y\|_\infty,
\]

where$L_{\text{att}} \le C_W + \frac{\gamma^2 C_W^3 M}{\sqrt{d}}.$
\end{lemma}

Through the Ledoux--Talagrand contraction inequality \citep{ledoux1991probability} and the previous Rademacher-complexity result, we obtain the following conclusion for $H$ layers.

\begin{lemma}[Depth-$H$ Rademacher complexity cascade]
Let $\mathcal{H}^{(H)} = \{ h \circ f_{H-1} \circ \cdots \circ f_1 : h \in \mathcal{H} \}$ denote the depth-$H$ composition of $H-1$ Lipschitz sparse attention layers (satisfying Lemma~4.4 with constant $L_{\text{att}}$) with the static sparse aggregation readout class $\mathcal{H}$. Then
\[
    \mathcal{R}_n(\mathcal{H}^{(H)}) \le (L_{\text{att}})^{(H-1)} \cdot \mathcal{R}_n(\mathcal{H}).
\]
\end{lemma}

\begin{remark} [Structural Stability]
In this work, we focus on the complexity of the sparse attention mechanism itself. While a full Transformer includes FFN and residual blocks, empirical evidence suggests that these components, particularly Layer Normalization, act as stabilizers that keep the effective Lipschitz constant of the entire block near unity (see \citealp{ba2016layer}). Therefore, our analysis of $(L_{\mathrm{att}})^H$ provides a worst-case upper bound and highlights the theoretical necessity of these architectural stabilizers in preventing generalization gap explosion; Section~\ref{sec:exp-rade} reports the empirical noise-fitting check.
\end{remark}

\begin{remark}[Scope and Limitations]
As in Remark~\ref{rem:fixed-support}, the above bound characterizes the generalization behavior of sparse models under a \textbf{fixed sparsity regime}, where the active support set is chosen independently of the input data, which aligns with classical structural risk control. Many practical mechanisms, such as Top-$k$ attention or routing-based Transformers, employ \textbf{data-dependent sparsity patterns}, where the support set varies with the input and the hypothesis class becomes input-dependent, so the standard Rademacher analysis may no longer directly apply. The present result should therefore be interpreted as a \emph{baseline complexity bound} that isolates the effect of sparsity itself, decoupled from additional complexity introduced by adaptive routing.
\end{remark}

\subsection{Information-Theoretic Perspective}

While Rademacher complexity provides insight into the dependence of
generalization on model sparsity, it does not directly capture the effect
of information loss induced by sparse attention. To bridge this gap, we
adopt an information-theoretic perspective. In particular, we leverage the
generalization bound of~\citet{xu2017information}, which
reduces the problem to characterizing the mutual information $I(S;W)$
between the training set $S$ and the learned hypothesis $W$. The key
challenge is therefore to understand how sparse attention affects the
information content of the learned model.

For the purpose of this analysis, we consider a discretized sparse-attention hypothesis class whose elements are parameterized by a support set $I\subseteq\{1,\dots,L\}$ with $|I|=M$ and $q\in\Delta^{M-1}$ is the sparse attention weight vector. The hypothesis $W=(I,q)$ therefore represents the sparse-attention mechanism considered in this bound, rather than the full parameter set of the Transformer model. Our derivation combines three ingredients: (i) the Xu--Raginsky bound, which relates the generalization error to $I(S;W)$; (ii) the entropy decomposition $I(S;W)\le H(W)=H(I)+H(q\mid I)$; and (iii) a covering-number bound on
the entropy of the sparse weight vector $q$. The complete proof is deferred to Appendix~\ref{app:proof-entropy-bound}.

\begin{theorem}[Entropy-Based Generalization Bound for Sparse Attention]
\label{thm:main_gen_bound}
Assume that:
\begin{enumerate}
\item the loss satisfies $\ell\in[0,1]$, so that
      $\sigma_\ell^2=1/4$;
\item the support set $I$ is data-independent choice over the family of $M$-subsets, so that the bound applies to the corresponding fixed-support sparse-attention hypothesis class;
\item the sparse attention weights $q\in\Delta^{M-1}$ are restricted to
      an $\epsilon$-net of the simplex under $\|\cdot\|_\infty$.
\end{enumerate}
Then, for any data distribution $\mu$,
\begin{equation}
\label{eq:main-entropy-bound}
\bigl|\mathrm{gen}(\mu,P_{W|S})\bigr|
\;\le\;
\sqrt{
\frac{M}{2n}
\left(
\log\frac{eL}{M}
+
\log\!\left(1+\frac{2}{\epsilon}\right)
\right)
}.
\end{equation}
\end{theorem}

\begin{remark}[Scope of the Bound]
\label{rem:scope}
Theorem~\ref{thm:main_gen_bound} is derived for a discretized sparse-attention hypothesis class with fixed, data-independent support (Remark~\ref{rem:fixed-support}). It does not directly characterize input-dependent routing or the full parameter space of a trained Transformer, and extending it to data-dependent sparse supports would require additional control of the routing mechanism.
\end{remark}

\begin{remark}[Interpretation of $\epsilon$]
\label{rem:epsilon}
The parameter \(\epsilon\) should be interpreted as the effective discretization or quantization resolution of the sparse attention weights. Thus, decreasing \(\epsilon\) enlarges the number of distinguishable attention-weight configurations and consequently loosens the bound. In a quantized implementation, \(\epsilon\) may be chosen according to the quantization resolution.
\end{remark}

\begin{remark}[Two Complementary Views]
\label{rem:two-views}
The JS divergence $f(\alpha)$ derived in
Theorem~\ref{thm:js_divergence} remains useful as a measure of
approximation fidelity, while Theorem~\ref{thm:main_gen_bound} provides the
generalization bound~\eqref{eq:main-entropy-bound}, which depends on the structural complexity $M$ and
the numerical resolution $\epsilon$. Together, they offer two
complementary views of sparse attention: $f(\alpha)$ quantifies how much
information is lost, and $M\log(1+2/\epsilon)$ quantifies how many
distinguishable hypotheses remain.
\end{remark}

\subsection{Synthesizing the Two Bounds}
\label{sec:synthesis}

The Rademacher and mutual-information bounds describe the same capacity reduction from different angles. The former is data-dependent and directly exposes the sample size, the sparse budget, the sequence length, and the layer-wise Lipschitz cascade. The latter is structural: it decomposes the hypothesis entropy into the number of possible supports and the resolution of the sparse weight simplex. Neither view subsumes the other, which is why both are used in the paper. Table~\ref{tab:two-bounds} summarizes the comparison.

\begin{table}[h]
\centering
\small
\begin{tabular}{>{\raggedright\arraybackslash}p{0.23\linewidth}%
                >{\raggedright\arraybackslash}p{0.34\linewidth}%
                >{\raggedright\arraybackslash}p{0.34\linewidth}}
\hline
 & \textbf{Rademacher bound} & \textbf{Mutual-information bound} \\
\hline
Primary object & Geometric covering and sample-dependent capacity & Entropy of support and sparse weights \\
\hline
Key dependence & $M$, $n$, $L$, and the depth factor $(L_{\mathrm{att}})^{H-1}$ & $M$, $n$, $L$, and the precision parameter $\epsilon$ \\
\hline
Strengths & Captures geometric capacity and layer-wise Lipschitz behavior & Separates discrete topology from continuous numerical resolution \\
\hline
Limitations & Does not separate topology from precision; fixed-support analysis & Requires a discretized weight class and a fixed, data-independent support \\
\hline
\end{tabular}
\caption{Complementary views of the capacity cost of sparse attention.}
\label{tab:two-bounds}
\end{table}

Together, the two bounds reveal a trade-off. When $M$ increases, the geometric and information-theoretic capacity terms grow, while the approximation cost $f(\alpha)$ tends to decrease because more attention mass is retained. When $M$ decreases, the capacity terms shrink, but the truncation mass can approach one and the approximation error saturates. This is the approximation--generalization tension that the remainder of the paper turns into design guidance.

\section{Prescriptive Architectural and Engineering Implications}
\label{sec:implications}

We now translate the bounds of Sections~\ref{sec:core} and~\ref{sec:generalization} into concrete guidance for the sparse budget, the head geometry, and the numerical precision of the sparse weights.

\subsection{Long-Sequence Sparse Budget}

If $M$ is constant, the prefactor $M/(L-M)$ in Theorem~\ref{thm:sparse_decay} decays like $1/L$, so the high-probability upper bound on $R$ tends to zero and the implied lower bound on $\alpha$ tends to one. A fixed budget therefore becomes fragile in long sequences. At the other extreme, growing $M$ too quickly enlarges both the Rademacher term $\sqrt{M/n}$ and the entropy terms. The scaling $M(L)=\mathcal{O}(\sqrt{L})$ balances these effects: it slows truncation growth relative to a constant budget while keeping the capacity terms sublinear in $L$.

\subsection{Multi-Head and Narrow-Head Preference}

Since $\|A_{ij}\|_{\psi_2}$ scales as $\gamma^2 C_W^2/\sqrt{\ln 2 \cdot d_{qk}}$, the head dimension controls the logit tail coefficient. In this analysis, smaller $d_{qk}$ is interpreted as slowing the growth of the truncation mass. At fixed $d_{\mathrm{model}}$, reducing $d_{qk}$ means increasing the number of heads $h=d_{\mathrm{model}}/d_{qk}$, which motivates a multi-head, narrower-per-head configuration subject to the entropy-collapse constraints on $\gamma$ and $C_W$.

\subsection{Sparse--Quantization Co-Design}

The parameter $\epsilon$ is the effective discretization or quantization resolution of the sparse attention weights and can be mapped to hardware grids such as FP8 or INT4. Lower precision means larger $\epsilon$, which reduces $\log(1+2/\epsilon)$ and tightens the information-theoretic capacity term. The effect is not uniform: when $M$ is very small, the approximation cost and the accuracy of the retained weights are central, so high precision is safer; when $M$ is larger, the topology term and sample-dependent factor dominate, and low-bit quantization changes the bound only weakly. Thus extreme sparsity calls for high precision, while a larger support budget can tolerate more aggressive quantization.

\section{Conclusion and Theoretical Frontiers}
\label{sec:conclusion}

We recast sparse attention as a structural inductive bias and organized its analysis around a cost--benefit trade-off. The JS divergence reduces approximation loss to the truncation mass $\alpha$; a mean-field analysis shows that truncation pressure grows with sequence length; Rademacher complexity quantifies the geometric capacity reduction and the layer-wise Lipschitz cascade; and a Xu--Raginsky decomposition separates discrete support selection from numerical precision.

The practical implications are that the sparse budget, head geometry, and precision should be designed jointly: constant budgets are fragile in long sequences, $\mathcal{O}(\sqrt{L})$ is a natural balance in these bounds, narrower heads are favored at fixed model width, and aggressive quantization is most defensible when the sparse support is already large. The main limitations are the mean-field, i.i.d. logit assumption, the fixed-support baseline, and the discretized weight simplex. Extending the analysis to weakly dependent logits and dynamic Top-$k$ routing remains future work. Sections~\ref{sec:exp-js} and~\ref{sec:exp-rade} support the predicted direction of information decay.

\section*{Acknowledgments}
We thank Halil Talha Ertan for his careful reading of an earlier version of this manuscript and for pointing out the issues in the formulation of the Rademacher complexity bound and the proof steps in Appendix~C.1.

\bibliographystyle{unsrtnat}
\bibliography{ref}

@inproceedings{vaswani2017attention,
  title={Attention is all you need},
  author={Vaswani, Ashish and Shazeer, Noam and Parmar, Niki and Uszkoreit, Jakob and Jones, Llion and Gomez, Aidan N and Kaiser, {\L}ukasz and Polosukhin, Illia},
  booktitle={Advances in Neural Information Processing Systems (NeurIPS)},
  volume={30},
  pages={5998--6008},
  year={2017},
  url={https://proceedings.neurips.cc/paper/2017/hash/3f5ee243547dee91fbd053c1c4a845aa-Abstract.html}
}

@article{child2019generating,
  title={Generating long sequences with sparse transformers},
  author={Child, Rewon and Gray, Scott and Radford, Alec and Sutskever, Ilya},
  journal={arXiv preprint arXiv:1904.10509},
  year={2019},
  doi={10.48550/arXiv.1904.10509},
  url={https://arxiv.org/abs/1904.10509}
}

@article{beltagy2020longformer,
  title={Longformer: The long-document transformer},
  author={Beltagy, Iz and Peters, Matthew E and Cohan, Arman},
  journal={arXiv preprint arXiv:2004.05150},
  year={2020},
  doi={10.48550/arXiv.2004.05150},
  url={https://arxiv.org/abs/2004.05150}
}

@inproceedings{kitaev2020reformer,
  title={Reformer: The efficient transformer},
  author={Kitaev, Nikita and Kaiser, {\L}ukasz and Levskaya, Anselm},
  booktitle={International Conference on Learning Representations (ICLR)},
  year={2020},
  url={https://openreview.net/forum?id=rkgNKkHtvB}
}

@inproceedings{zhang2017understanding,
  title={Understanding deep learning requires rethinking generalization},
  author={Zhang, Chiyuan and Bengio, Samy and Hardt, Moritz and Recht, Benjamin and Vinyals, Oriol},
  booktitle={International Conference on Learning Representations (ICLR)},
  year={2017},
  url={https://openreview.net/forum?id=Sy8gdB9xx}
}

@article{lin1991divergence,
  title={Divergence measures based on the Shannon entropy},
  author={Lin, J.},
  journal={IEEE Transactions on Information Theory},
  volume={37},
  number={1},
  pages={145--151},
  year={1991},
  doi={10.1109/18.61115},
  url={https://doi.org/10.1109/18.61115}
}

@book{stein2011functional,
  title={Functional Analysis: Introduction to Further Topics in Analysis},
  author={Stein, Elias M. and Shakarchi, Rami},
  publisher={Princeton University Press},
  year={2011},
  series={Princeton Lectures in Analysis},
  volume={4},
  note={Theorem 1.1 (Hölder) appears in Chapter 1},
  isbn={9780691113876}
}

@inproceedings{xu2017information,
  title={Information-theoretic analysis of generalization capability of learning algorithms},
  author={Xu, Aolin and Raginsky, Maxim},
  booktitle={Advances in Neural Information Processing Systems (NeurIPS)},
  volume={30},
  pages={2524--2533},
  year={2017},
  url={https://proceedings.neurips.cc/paper_files/paper/2017/hash/ad71c82b22f4f65b9398f76d8be4c615-Abstract.html}
}

@article{yun2020n,
  title={O (n) connections are expressive enough: Universal approximability of sparse transformers},
  author={Yun, Chulhee and Chang, Yin-Wen and Bhojanapalli, Srinadh and Rawat, Ankit Singh and Reddi, Sashank and Kumar, Sanjiv},
  journal={Advances in Neural Information Processing Systems},
  volume={33},
  pages={13783--13794},
  year={2020},
  url={https://proceedings.neurips.cc/paper/2020/hash/9ed27554c893b5bad850a422c3538c15-Abstract.html}
}

@article{likhosherstov2021expressive,
  title={On the expressive power of self-attention matrices},
  author={Likhosherstov, Valerii and Choromanski, Krzysztof and Weller, Adrian},
  journal={arXiv preprint arXiv:2106.03764},
  year={2021},
  doi={10.48550/arXiv.2106.03764},
  url={https://arxiv.org/abs/2106.03764}
}

@book{vershynin2020high,
  title={High-Dimensional Probability: An Introduction with Applications in Data Science},
  author={Vershynin, Roman},
  publisher={Cambridge University Press},
  series={Cambridge Series in Statistical and Probabilistic Mathematics},
  year={2018},
  doi={10.1017/9781108231596},
  isbn={9781108415194}
}

@inproceedings{yuan2025native,
  title={Native sparse attention: Hardware-aligned and natively trainable sparse attention},
  author={Yuan, Jingyang and Gao, Huazuo and Dai, Damai and Luo, Junyu and Zhao, Liang and Zhang, Zhengyan and Xie, Zhenda and Wei, Yuxing and Wang, Lean and Xiao, Zhiping and others},
  booktitle={Proceedings of the 63rd Annual Meeting of the Association for Computational Linguistics (Volume 1: Long Papers)},
  pages={23078--23097},
  year={2025},
  doi={10.18653/v1/2025.acl-long.1126},
  url={https://aclanthology.org/2025.acl-long.1126/}
}

@article{roy2021efficient,
  title={Efficient content-based sparse attention with routing transformers},
  author={Roy, Aurko and Saffar, Mohammad and Vaswani, Ashish and Grangier, David},
  journal={Transactions of the Association for Computational Linguistics},
  volume={9},
  pages={53--68},
  year={2021},
  doi={10.1162/tacl_a_00353},
  url={https://doi.org/10.1162/tacl_a_00353},
  eprint={2003.05997},
}

@article{wang2020linformer,
  title={Linformer: Self-attention with linear complexity},
  author={Wang, Sinong and Li, Belinda Z and Khabsa, Madian and Fang, Han and Ma, Hao},
  journal={arXiv preprint arXiv:2006.04768},
  year={2020},
  doi={10.48550/arXiv.2006.04768},
  url={https://arxiv.org/abs/2006.04768}
}

@inproceedings{zhai2023stabilizing,
  title={Stabilizing Transformer Training by Preventing Attention Entropy Collapse},
  author={Zhai, Shuangfei and Likhomanenko, Tatiana and Littwin, Etai and Busbridge, Dan and Ramapuram, Jason and Zhang, Yizhe and Gu, Jiatao and Susskind, Joshua M},
  booktitle={Proceedings of the 40th International Conference on Machine Learning},
  series={Proceedings of Machine Learning Research},
  volume={202},
  pages={40770--40803},
  year={2023},
  organization={PMLR},
  url={https://proceedings.mlr.press/v202/zhai23a.html},
  eprint={2303.06296},
  archivePrefix={arXiv}
}

@inproceedings{qi2025taming,
  title={Taming transformer without using learning rate warmup},
  author={Qi, Xianbiao and He, Yelin and Ye, Jiaquan and Li, Chun-Guang and Zi, Bojia and Dai, Xili and Zou, Qin and Xiao, Rong},
  booktitle={International Conference on Learning Representations (ICLR)},
  year={2025},
  url={https://openreview.net/forum?id=GeUK3zGreN}
}

@inproceedings{zaheer2021bigbirdtransformerslonger,
      title={Big Bird: Transformers for Longer Sequences}, 
      author={Manzil Zaheer and Guru Guruganesh and Avinava Dubey and Joshua Ainslie and Chris Alberti and Santiago Ontanon and Philip Pham and Anirudh Ravula and Qifan Wang and Li Yang and Amr Ahmed},
      booktitle={Advances in Neural Information Processing Systems (NeurIPS)},
      year={2020},
      eprint={2007.14062},
      archivePrefix={arXiv},
      primaryClass={cs.LG},
      url={https://proceedings.neurips.cc/paper/2020/hash/c8512d142a2d849725f31a9a7a361ab9-Abstract.html},
}

@inproceedings{chen2024theoreticallimitationsmultilayertransformer,
      title={Theoretical limitations of multi-layer Transformer}, 
      author={Lijie Chen and Binghui Peng and Hongxun Wu},
      booktitle={IEEE 66th Annual Symposium on Foundations of Computer Science (FOCS)},
      pages={2631--2653},
      year={2025},
      eprint={2412.02975},
      archivePrefix={arXiv},
      primaryClass={cs.LG},
      doi={10.1109/FOCS63196.2025.00136},
      url={https://doi.org/10.1109/FOCS63196.2025.00136},
}

@inproceedings{hjelm2019learningdeeprepresentationsmutual,
      title={Learning deep representations by mutual information estimation and maximization}, 
      author={R Devon Hjelm and Alex Fedorov and Samuel Lavoie-Marchildon and Karan Grewal and Phil Bachman and Adam Trischler and Yoshua Bengio},
      booktitle={International Conference on Learning Representations (ICLR)},
      year={2019},
      eprint={1808.06670},
      archivePrefix={arXiv},
      primaryClass={stat.ML},
      url={https://openreview.net/forum?id=Bklr3j0cKX},
}

@inproceedings{dorent2025connectingjensenshannonkullbackleiblerdivergences,
      title={Connecting Jensen-Shannon and Kullback-Leibler Divergences: A New Bound for Representation Learning}, 
      author={Reuben Dorent and Polina Golland and William Wells III},
      booktitle={Advances in Neural Information Processing Systems (NeurIPS)},
      pages={53913--53946},
      year={2025},
      eprint={2510.20644},
      archivePrefix={arXiv},
      primaryClass={cs.LG},
      doi={10.52202/085713-1614},
      url={https://doi.org/10.52202/085713-1614},
}

@article{dudley1967sizes,
title={The sizes of compact subsets of Hilbert space and continuity of Gaussian processes},
author={Dudley, Richard M},
journal={Journal of Functional Analysis},
volume={1},
number={3},
pages={290--330},
year={1967},
publisher={Academic Press},
doi={10.1016/0022-1236(67)90017-1},
url={https://doi.org/10.1016/0022-1236(67)90017-1}
}

@article{bartlett2002rademacher,
  title={Rademacher and Gaussian complexities: Risk bounds and structural results},
  author={Bartlett, Peter L and Mendelson, Shahar},
  journal={Journal of Machine Learning Research},
  volume={3},
  number={Nov},
  pages={463--482},
  year={2002},
  url={https://www.jmlr.org/papers/v3/bartlett02a.html}
}

@book{ledoux1991probability,
  title={Probability in Banach Spaces: Isoperimetry and Processes},
  author={Ledoux, Michel and Talagrand, Michel},
  year={1991},
  publisher={Springer-Verlag},
  doi={10.1007/978-3-642-20212-4},
  url={https://link.springer.com/book/10.1007/978-3-642-20212-4}
}

@inproceedings{trauger2023sequence,
  title={Sequence length independent norm-based generalization bounds for transformers},
  author={Trauger, Jacob and Tewari, Ambuj},
  booktitle={International Conference on Artificial Intelligence and Statistics (AISTATS)},
  series={Proceedings of Machine Learning Research},
  volume={238},
  pages={1405--1413},
  year={2024},
  url={https://proceedings.mlr.press/v238/trauger24a.html}
}

@inproceedings{zhang2026generalization,
  title={Generalization Bounds for Transformer Channel Decoders},
  author={Zhang, Qinshan and Chen, Bin and Jiang, Yong and Xia, Shu-Tao},
  booktitle={IEEE International Symposium on Information Theory (ISIT)},
  pages={1--6},
  year={2026},
  doi={10.1109/ISIT62367.2026.11653954},
  url={https://doi.org/10.1109/ISIT62367.2026.11653954}
}

@article{ba2016layer,
  title={Layer normalization},
  author={Ba, Jimmy Lei and Kiros, Ryan and Hinton, Geoffrey E},
  journal={arXiv preprint arXiv:1607.06450},
  year={2016},
  url={https://arxiv.org/abs/1607.06450}
}

@article{nair2026softmax,
  title={Softmax is 1/2-Lipschitz: A tight bound across all $\ell_p$ norms},
  author={Nair, Pravin},
  journal={Transactions on Machine Learning Research},
  year={2026},
  url={https://arxiv.org/abs/2510.23012}
}

@inproceedings{sanford2023representational,
  author    = {Clayton Sanford and Daniel Hsu and Matus Telgarsky},
  title     = {Representational Strengths and Limitations of Transformers},
  booktitle = {Advances in Neural Information Processing Systems (NeurIPS)},
  volume    = {36},
  pages     = {72947--72980},
  year      = {2023},
  url       = {https://proceedings.neurips.cc/paper_files/paper/2023/hash/73bf692447f174984f30499ec9b20e04-Abstract-Conference.html}
}

@inproceedings{chu2023unified,
  author    = {Yifeng Chu and Maxim Raginsky},
  title     = {A Unified Framework for Information-Theoretic Generalization Bounds},
  booktitle = {Advances in Neural Information Processing Systems (NeurIPS)},
  volume    = {36},
  pages     = {40656--40675},
  year      = {2023},
  url       = {https://proceedings.neurips.cc/paper_files/paper/2023/hash/fa67d13ba6c73637593bbcc92f6400ff-Abstract-Conference.html}
}

@inproceedings{nadjahi2024slicing,
  author    = {Kimia Nadjahi and Kristjan Greenewald and Rickard Br{\"u}el Gabrielsson and Justin Solomon},
  title     = {Slicing Mutual Information Generalization Bounds for Neural Networks},
  booktitle = {Proceedings of the 41st International Conference on Machine Learning (ICML)},
  volume    = {235},
  pages     = {37213--37236},
  year      = {2024},
  url       = {https://proceedings.mlr.press/v235/nadjahi24a.html}
}

@inproceedings{havrilla2024understanding,
  author    = {Alex Havrilla and Wenjing Liao},
  title     = {Understanding Scaling Laws with Statistical and Approximation Theory for Transformer Neural Networks on Intrinsically Low-dimensional Data},
  booktitle = {Advances in Neural Information Processing Systems (NeurIPS)},
  volume    = {37},
  pages     = {126938--126987},
  year      = {2024},
  url       = {https://neurips.cc/virtual/2024/poster/95466}
}

\appendix

\section{JS Divergence Analysis}\label{app:A}
This appendix collects the proofs and the auxiliary analysis supporting the closed-form divergence of Section~\ref{sec:core}.

\subsection{Proof of Lemma~\ref{lemma:weight_scaling}}\label{app:A1}
\begin{proof}
We start from the definition of the full attention weight for $j \in \mathcal{S}_i$:
\[
p_i(j) = \frac{\exp(S_{ij})}{\sum_{k=1}^L \exp(S_{ik})} = \frac{\exp(S_{ij})}{\sum_{k \in \mathcal{S}_i} \exp(S_{ik}) + \sum_{k \notin \mathcal{S}_i} \exp(S_{ik})}
\]
By the definition of the truncation mass $\alpha_i$:
\[
\alpha_i = \sum_{k \notin \mathcal{S}_i} p_i(k) = \frac{\sum_{k \notin \mathcal{S}_i} \exp(S_{ik})}{\sum_{k=1}^L \exp(S_{ik})}
\]
Rearranging gives the relationship between the partition functions:
\[
\sum_{k \notin \mathcal{S}_i} \exp(S_{ik}) = \frac{\alpha_i}{1-\alpha_i} \sum_{k \in \mathcal{S}_i} \exp(S_{ik})
\]
Substituting this back into the expression for $p_i(j)$:
\[
p_i(j) = \frac{\exp(S_{ij})}{\sum_{k \in \mathcal{S}_i} \exp(S_{ik}) + \frac{\alpha_i}{1-\alpha_i} \sum_{k \in \mathcal{S}_i} \exp(S_{ik})} = (1-\alpha_i) \cdot \frac{\exp(S_{ij})}{\sum_{k \in \mathcal{S}_i} \exp(S_{ik})}
\]
By the definition of the sparse attention weight $q_i(j)$, this simplifies to:
\[
p_i(j) = (1-\alpha_i) q_i(j)
\]
for all $j \in \mathcal{S}_i$.

\end{proof}

\subsection{Proof of Theorem~\ref{thm:js_divergence}}\label{app:A2}
\begin{proof}
We first compute $\KL(q_i \| m)$. For $j \in \mathcal{S}_i^c$, $q_i(j) = 0$, so these terms contribute 0 to the KL divergence. For $j \in \mathcal{S}_i$:
\[
\frac{q_i(j)}{m(j)} = \frac{q_i(j)}{q_i(j) \cdot (2-\alpha)/2} = \frac{2}{2-\alpha}
\]
Summing over the support set (where $\sum_{j \in \mathcal{S}_i} q_i(j) = 1$):
\[
\KL(q_i \| m) = \sum_{j \in \mathcal{S}_i} q_i(j) \ln\left( \frac{2}{2-\alpha} \right) = \ln\left( \frac{2}{2-\alpha} \right)
\]
Substitute the KL divergence results into the JS divergence definition:
\[
\JS(\alpha) = \frac{1}{2}\left[ (1-\alpha) \ln\left( \frac{2(1-\alpha)}{2-\alpha} \right) + \alpha \ln 2 + \ln\left( \frac{2}{2-\alpha} \right) \right]
\]
Expand and simplify the logarithmic terms provide:
\[
\JS(\alpha) = \ln 2 + \frac{1-\alpha}{2} \ln(1-\alpha) - \frac{2-\alpha}{2} \ln(2-\alpha)
\]
This completes the derivation of the closed-form expression. We now prove strict monotonicity and convexity via the first and second derivatives:

\[
\JS'(\alpha) = \frac{d}{d\alpha}\left[ \ln 2 + \frac{1-\alpha}{2}\ln(1-\alpha) - \frac{2-\alpha}{2}\ln(2-\alpha) \right] = \frac{1}{2}\ln\left( \frac{2-\alpha}{1-\alpha} \right)
\]
For all $\alpha \in (0,1)$, $\frac{2-\alpha}{1-\alpha} > 1$, so $\JS'(\alpha) > 0$. Thus, $\JS(\alpha)$ is strictly increasing on $(0,1)$.

\[
\JS''(\alpha) = \frac{d}{d\alpha}\left[ \frac{1}{2}\ln(2-\alpha) - \frac{1}{2}\ln(1-\alpha) \right] = \frac{1}{2(1-\alpha)(2-\alpha)}
\]
For all $\alpha \in (0,1)$, $(1-\alpha)(2-\alpha) > 0$, so $\JS''(\alpha) > 0$. Thus, $\JS(\alpha)$ is strictly convex on $(0,1)$, completing the proof.
\end{proof}
\subsection{JS Divergence Graph Analysis and Description}\label{app:A3}
\begin{figure}[!ht]
    \centering
    \includegraphics[width=0.8\textwidth]{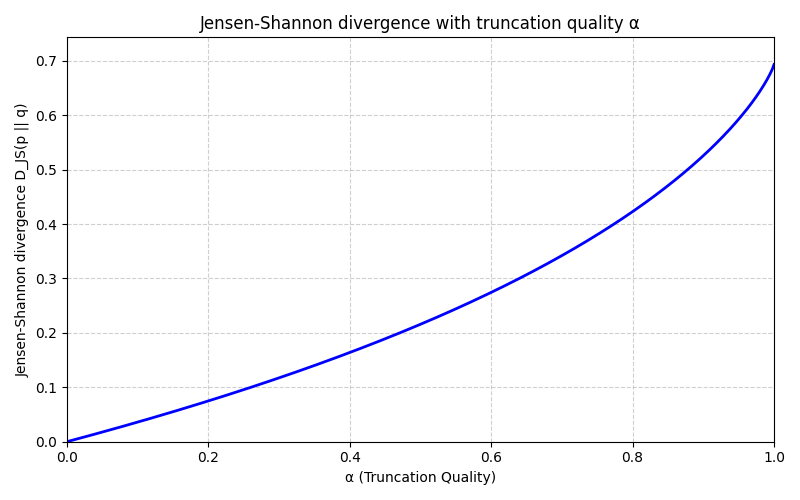}
    \caption{The JS divergence $f(\alpha)$ as a function of the discarded mass $\alpha$.}
    \label{fig:js_with_alpha}
\end{figure}

According to Theorem~\ref{thm:js_divergence}, for each query position, the Jensen--Shannon divergence between the full attention distribution \(p_i\) and the sparse distribution \(q_i\) depends solely on the truncation mass \(\alpha_i = \sum_{j\notin \mathcal{S}_i} p_i(j)\):
\[
D_{\mathrm{JS}}(p_i \| q_i) = \ln 2 + \frac{1-\alpha_i}{2}\ln(1-\alpha_i) - \frac{2-\alpha_i}{2}\ln(2-\alpha_i),\quad \alpha_i\in[0,1].
\]
The function \(f(\alpha)=D_{\mathrm{JS}}\) is strictly convex and monotonically increasing on \([0,1)\), as illustrated in Figure~\ref{fig:js_with_alpha}. It satisfies \(f(0)=0\) and \(\lim_{\alpha\to 1^-}f(\alpha)=\ln2\). Applying Jensen's inequality to a set of loss masses \(\alpha_1,\dots,\alpha_n\) yields
\[
\frac{1}{n}\sum_{i=1}^n D_{\mathrm{JS}}(p_i\|q_i) \ge f\!\left(\frac{1}{n}\sum_{i=1}^n\alpha_i\right),
\]
which shows that the average JS divergence is bounded below by the divergence evaluated at the average loss mass, and the minimum is achieved when all \(\alpha_i\) are equal.

\section{Details and Additional Analysis for the Truncation Mass \texorpdfstring{$\alpha$}{alpha}}\label{app:B}
This appendix collects the supporting material for the concentration analysis of Section~\ref{sec:core}: the proof of the sub-Gaussian tail bound, the empirical norm statistics behind Assumptions~\ref{ass:bounded_spectral_norm} and~\ref{ass:bounded_embedding_norm}, the probability tools, the order-statistics bound, and the proof of the main concentration result.

\subsection{Proof of Proposition~\ref{prop:subgaussian_attention}}\label{app:proof_subgaussian}
\begin{lemma}[\citealp{vershynin2020high}]
Let $X$ be a random variable with $\|X\|_\infty < \infty$ almost surely. Then
\[
\|X\|_{\psi_2} \leq \frac{\|X\|_\infty}{\sqrt{\ln 2}}.
\]
\end{lemma}
Back to the proof of Proposition~\ref{prop:subgaussian_attention}.
\begin{proof}
We first bound the absolute value of $A_{ij}$ almost surely. By the Cauchy-Schwarz inequality and spectral norm properties:
\[
|A_{ij}| = \frac{1}{\sqrt{d_{qk}}} \left| \left(X_i W^Q\right)^\top \left(X_j W^K\right) \right| \leq \frac{1}{\sqrt{d_{qk}}} \left\| X_i W^Q \right\|_2 \left\| X_j W^K \right\|_2
\]
 We have $\|X_i\|_2 \leq \gamma$ and $\|X_j\|_2 \leq\gamma$. By the definition of spectral norm:
\[
\left\| X_i W^Q \right\|_2 \leq \left\| W^Q \right\|_2 \left\| X_i \right\|_2 \leq \gamma C_W, \quad \left\| X_j W^K \right\|_2 \leq \left\| W^K \right\|_2 \left\| X_j \right\|_2 \leq \gamma C_W
\]
Substituting these bounds gives the following:
\[
\|A_{ij}\|_{\psi_2} \leq \frac{\gamma^2 C_W^2}{\sqrt{\ln 2 \cdot d_{qk}}}
\]
\end{proof}

\subsection{Norm Statistics of the Projection Weights and of the Embeddings}
\label{app:norm_stats}
The bounds assumed in Assumption~\ref{ass:bounded_spectral_norm} and Assumption~\ref{ass:bounded_embedding_norm} are verified on \texttt{bert-base-uncased}. The spectral norms of the query and key projection matrices are distributed in the overall range $1.5$--$2.8$, where $W_Q$ spans a wider range and reaches a higher upper bound, whereas $W_K$ is more concentrated (Figure~\ref{fig:proj_norm}). The $L_2$ norms of the word embeddings on texts from \texttt{WikiText-103} lie in the overall range $0$--$2.5$, with a core concentration between $0.5$ and $2.0$: longer texts concentrate near $1.0$, while shorter texts spread more widely and contain abundant high-norm outliers (Figure~\ref{fig:vob_emd}).

\begin{figure}[!ht]
\centering
\subfigure[Spectral norm (median)]{%
  \includegraphics[width=0.48\linewidth]{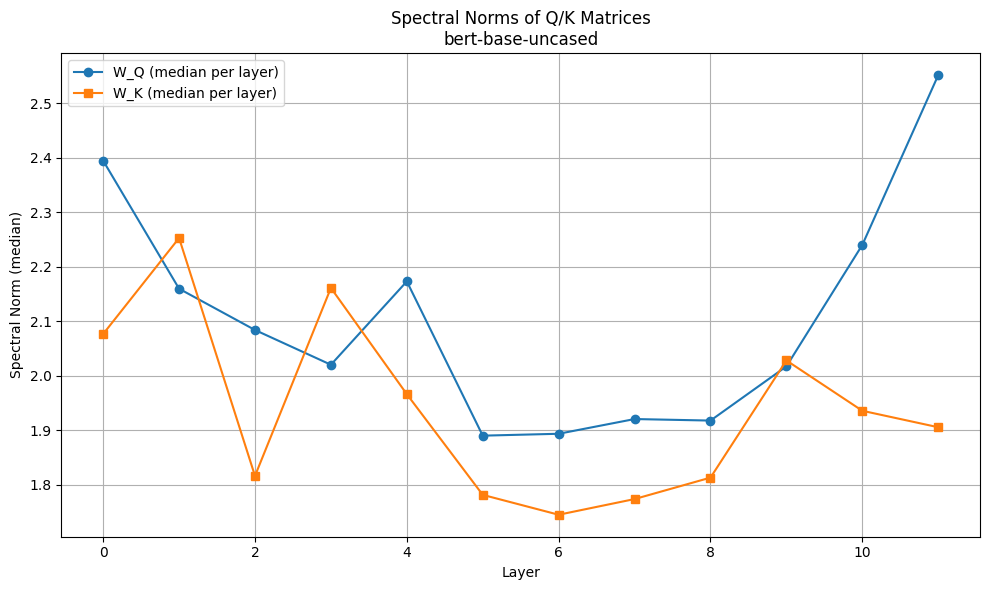}%
  \label{fig:proj_norm_1}%
}
\hfill
\subfigure[Spectral norm distribution]{%
  \includegraphics[width=0.48\linewidth]{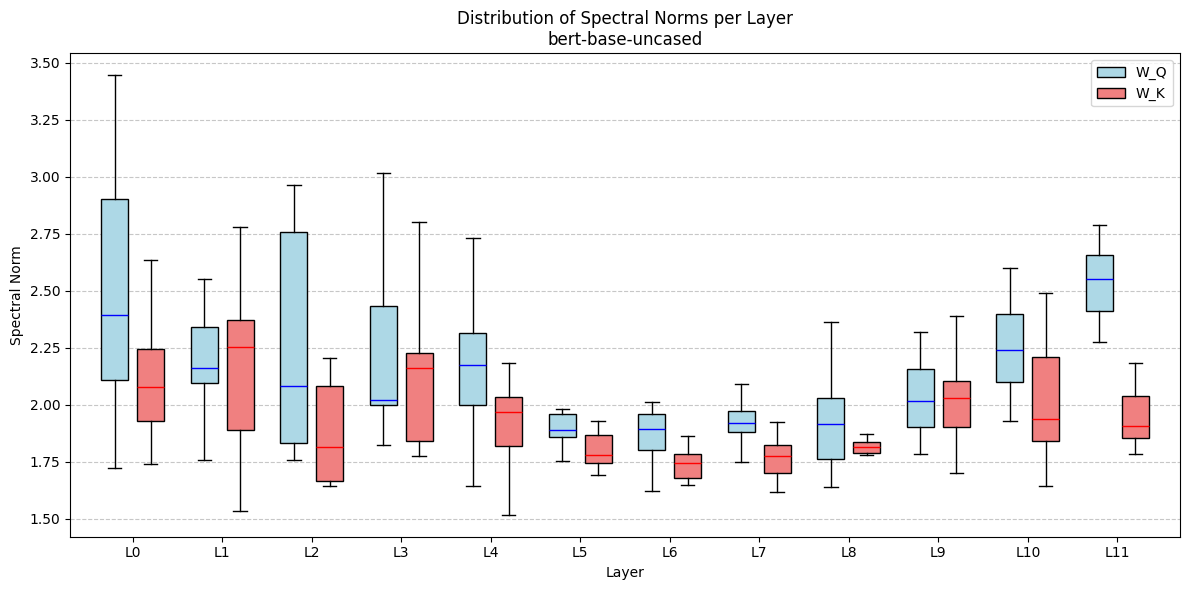}%
  \label{fig:proj_norm_2}%
}
\caption{Distribution of BERT self-attention Q/K spectral norms.}
\label{fig:proj_norm}
\end{figure}

\begin{figure}[!ht]
\centering
\subfigure[Embedding norm distribution (1)]{%
  \includegraphics[width=0.48\linewidth]{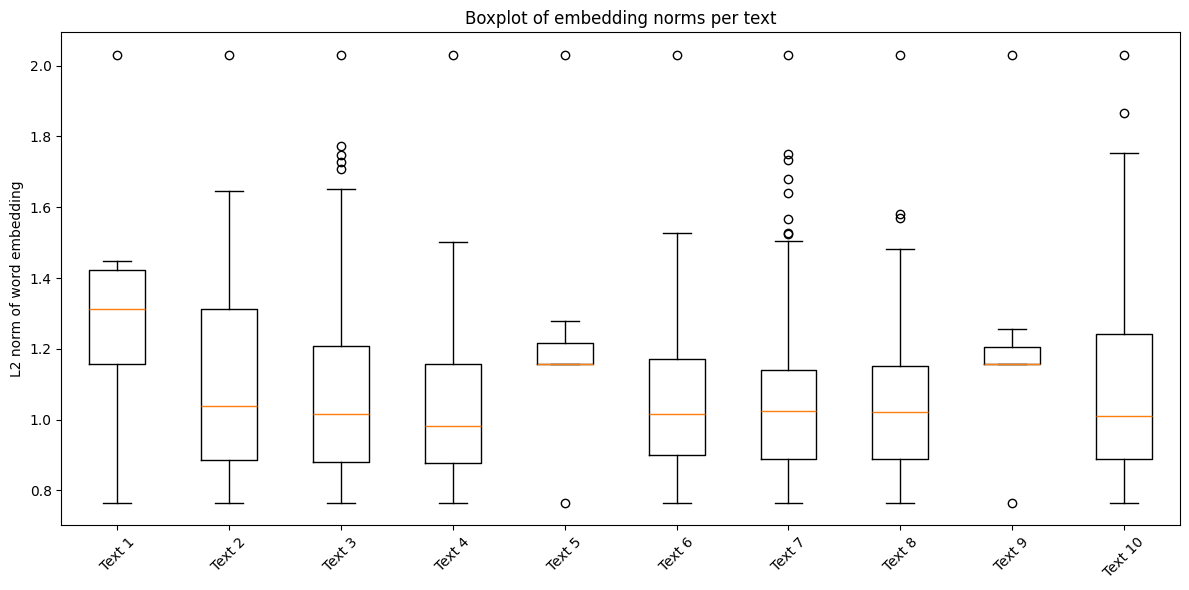}%
  \label{fig:vob_1}%
}
\hfill
\subfigure[Embedding norm distribution (2)]{%
  \includegraphics[width=0.48\linewidth]{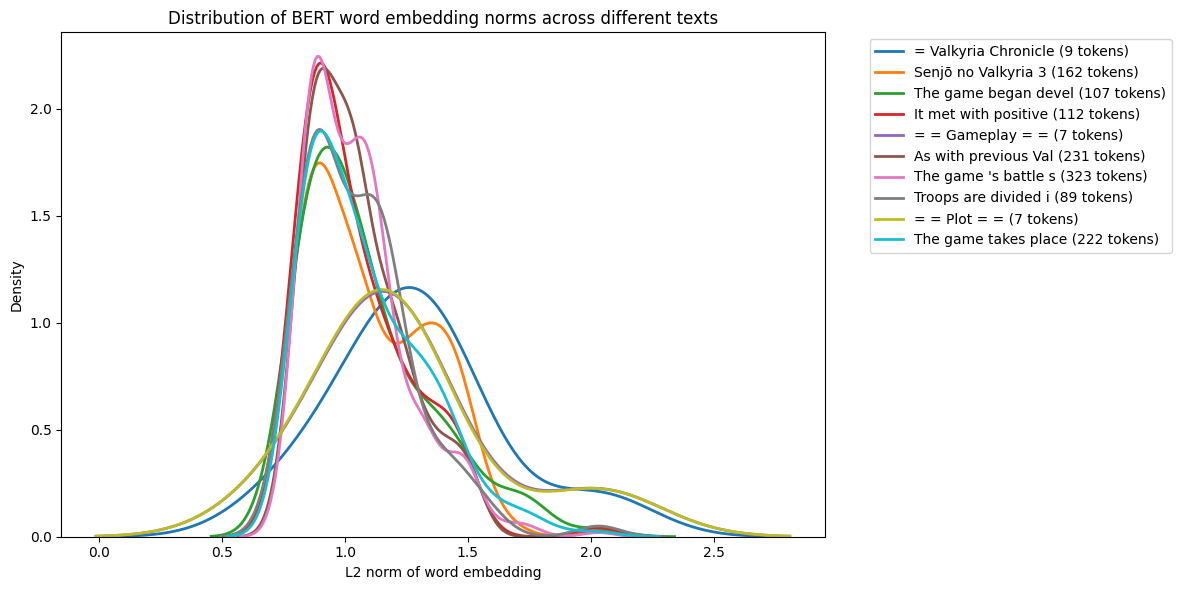}%
  \label{fig:vob_embed_2}%
}
\caption{Distribution of BERT word-embedding norms.}
\label{fig:vob_emd}
\end{figure}

\subsection{Some Probability Techniques}\label{app:prob_tools}
\begin{lemma}[Bounded random variables are sub-Gaussian]
If $|X| \leq M$ almost surely for some $M>0$, then the random variable $X$ is sub-Gaussian with parameter $\sigma^2 = M^2$ \citep{vershynin2020high}.
\label{lemma:bounded_subgaussian}
\end{lemma}

\begin{lemma}[Sub-Gaussian Tail Bound, \citealp{vershynin2020high}]
\label{lemma:subgaussian_tail}
Let $X$ be a zero-mean sub-Gaussian random variable with parameter $\sigma^2$. For any $u>0$, the upper tail bound holds:
\[
\Prob(X \geq u) \leq \exp\left( -\frac{u^2}{2\sigma^2} \right)
\]
\end{lemma}
\begin{lemma}[Chernoff bound for the binomial distribution]
\label{lem:chernoff_bound}
Let \( X_1, X_2, \dots, X_n \) be \( n \) independent Bernoulli random variables with \( \Prob(X_i = 1) = p \), \( \Prob(X_i = 0) = 1-p \), \( 0 < p < 1 \). Define \( X = X_1 + \dots + X_n \). Then
\[
\Prob(X \geq m) = \sum_{k=m}^n \binom{n}{k} p^k (1-p)^{n-k},
\]
and for all \( t \geq 0 \) and \( m > np \),
\[
\Prob(X \geq m) \leq \exp\left( -tm + np\left(e^t - 1\right) \right).
\]
In particular,
\[
\Prob(X \geq m) \leq \left( \frac{np}{m} \right)^m e^{m - np} \leq \left( \frac{e np}{m} \right)^m.
\]
\end{lemma}

\begin{proof}
\[
\Prob(X \geq m) = \E\left[\ind{X \geq m}\right] \leq \E\left[e^{t(X - m)}\right] = e^{-tm} \E\left[e^{tX}\right].
\]
For each Bernoulli variable \( X_i \),
\[
\E\left[e^{tX_i}\right] = (1-p)e^{0} + p e^{t} = 1 - p + p e^t.
\]
By independence,
\[
\E\left[e^{tX}\right] = \prod_{i=1}^n \E\left[e^{tX_i}\right] = \left( 1 - p + p e^t \right)^n \leq \left(e^{p(e^t - 1)}\right)^n = e^{np(e^t - 1)}.
\]
Thus,
\[
\Prob(X \geq m) \leq e^{-tm + np(e^t - 1)} = \exp\left( -tm + np\left(e^t - 1\right) \right).
\]
Let \( g(t) = -tm + np(e^t - 1) \). Take the derivative and set to zero to find the minimizer:
\[
g'(t) = -m + np e^t = 0 \implies t^* = \ln\left( \frac{m}{np} \right).
\]
Substitute \( t^* \) back:
\[
\Prob(X \geq m) \leq \exp\left(g(t^*)\right) = \left( \frac{np}{m} \right)^m e^{m - np} \leq \left( \frac{e np}{m} \right)^m.
\]
For \( m \leq np \), \( \Prob(X \geq m) \leq 1 \leq \left( \frac{e np}{m} \right)^m \), so the bound holds for all \( m \).

\end{proof}
\subsection{Order Statistics Analysis}\label{app:order_stats}
\begin{definition}[Order Statistics]
Let $S_{(1)} \geq S_{(2)} \geq \dots \geq S_{(L)}$ denote the descending order statistics of the attention scores $S_{i1}, \dots, S_{iL}$, where $S_{(k)}$ is the $k$-th largest value in the sample.
\label{def:order_statistics}
\end{definition}

\begin{theorem}[Order statistic difference bound]
\label{thm:order_stat_difference}
Let the attention scores $S_{i1}, \dots, S_{iL}$ be independent and identically distributed sub-Gaussian random variables with sub-Gaussian norm at most $\sigma$ (Assumption~\ref{ass:independent_attention_scores}), and let $S_{(u)}$ and $S_{(v)}$ be the order statistics of Definition~\ref{def:order_statistics} with $1 \leq u < v \leq L$. Then, for any $\delta \in (0,1)$,
\[
\Prob\left( S_{(u)} - S_{(v)} \leq  \sigma \sqrt{2}\left( \sqrt{\ln \frac{eL}{u}} + \sqrt{\ln \frac{eL}{L-v+1}}  \right)+\sqrt{2}\sigma+4 \sigma \sqrt{ \ln \frac{2}{\delta} } \right) \geq 1 - \delta
\]
\end{theorem}

\begin{proof}
We first formalize the core notations and definitions for our derivation:
let $S_{(1)} \geq S_{(2)} \geq \dots \geq S_{(L)}$ be the descending order statistics of $L$ independent and identically distributed sub-Gaussian random variables with sub-Gaussian norm at most $\sigma$. For any indices $1 \leq u < v \leq L$, we define:
\begin{align*}
D_{u, v} &:= S_{(u)} - S_{(v)}, \\
\mu_{u, v} &:= \mathbb{E}\left[S_{(u)}\right] - \mathbb{E}\left[S_{(v)}\right] = \mathbb{E}\left[D_{u, v}\right], \\
\Delta_u &:= S_{(u)} - \mathbb{E}\left[S_{(u)}\right], \quad \Delta_v := S_{(v)} - \mathbb{E}\left[S_{(v)}\right].
\end{align*}
By definition, we have the decomposition:
\[
D_{u, v} = \Delta_u - \Delta_v + \mu_{u, v}.
\]
Taking expectation on both sides confirms $\mathbb{E}\left[D_{u, v}\right] = \mu_{u, v}$, as $\mathbb{E}[\Delta_u] = \mathbb{E}[\Delta_v] = 0$.
$\mu_{u,v} = \mathbb{E}[S_{(u)}] - \mathbb{E}[S_{(v)}]$, we rearrange terms to get:
\[
\mu_{u, v} = \int_{0}^{\infty} \left( \Prob(S_{(u)} > t) - \Prob(S_{(v)} > t) \right) dt + \int_{0}^{\infty} \left( \Prob(S_{(v)} < -t) - \Prob(S_{(u)} < -t) \right) dt.
\]
By the property of descending order statistics, $S_{(u)} \geq S_{(v)}$ holds almost surely for $u < v$, so the integrands above are non-negative.
We first bound the integral of the survival function for $S_{(u)}$. For a threshold $t_u$ to be specified, split the integral:
\[
\int_{0}^{\infty} \Prob(S_{(u)} > t) dt = \int_{0}^{t_u} \Prob(S_{(u)} > t) dt + \int_{t_u}^{\infty} \Prob(S_{(u)} > t) dt.
\]
For the first term, we use the trivial bound $\Prob(S_{(u)} > t) \leq 1$, so $\int_{0}^{t_u} \Prob(S_{(u)} > t) dt \leq t_u$.
Note that by Lemma~\ref{lem:chernoff_bound}:
\[
\Prob(S_{(u)} > t) \leq \left( \frac{e L}{u} \right)^u \exp\left( -\frac{u t^2}{2 \sigma^2} \right).
\]
We choose $t_u$ such that the leading term equals 1, that is, $\left( \frac{e L}{u} \right)^u \exp\left( -\frac{u t_u^2}{2 \sigma^2} \right) = 1$. Solving for $t_u$ gives:
\[
t_u = \sigma \sqrt{2 \ln\left( \frac{e L}{u} \right)}.
\]
We then perform the change of variables $s = \frac{\sqrt{u} t}{\sigma}$ (so $dt = \frac{\sigma}{\sqrt{u}} ds$) and apply the Mill's ratio bound $\int_{a}^{\infty} e^{-s^2/2} ds \leq \frac{1}{a} e^{-a^2/2}$ for $a>0$:
\begin{align*}
\int_{t_u}^{\infty} \Prob(S_{(u)} > t) dt &\leq \left( \frac{e L}{u} \right)^u \int_{t_u}^{\infty} \exp\left( -\frac{u t^2}{2 \sigma^2} \right) dt \\
&= \frac{\sigma}{\sqrt{u}} \left( \frac{e L}{u} \right)^u \int_{\frac{\sqrt{u} t_u}{\sigma}}^{\infty} e^{-s^2/2} ds \\
&\leq \frac{\sigma}{\sqrt{u}} \cdot \frac{\sigma}{\sqrt{2 \ln\left( \frac{e L}{u} \right)}} \exp\left( -\ln\left( \frac{e L}{u} \right) \right) \\
&\leq \frac{\sigma}{\sqrt{2}}.
\end{align*}
Combining the two terms, we get:
\[
\int_{0}^{\infty} \Prob(S_{(u)} > t) dt \leq \sigma \sqrt{2 \ln\left( \frac{e L}{u} \right)} + \frac{\sigma}{\sqrt{2}}.
\]
By symmetry, the negative tail integral satisfies:
\[
\int_{0}^{\infty} \Prob(S_{(v)} < -t) dt \leq \sigma \sqrt{2 \ln\left( \frac{e L}{L - v + 1} \right)} + \frac{\sigma}{\sqrt{2}}.
\]
Substituting back into the integral representation of $\mu_{u,v}$, we obtain the final upper bound for the expectation:
\[
\mu_{u, v} \leq \sigma \sqrt{2} \left( \sqrt{\ln\left( \frac{e L}{u} \right)} + \sqrt{\ln\left( \frac{e L}{L - v + 1} \right)} \right) + \sqrt{2} \sigma.
\]
Define the zero-mean centered random variable:
\[
Z := D_{u, v} - \mu_{u, v} = \Delta_u - \Delta_v.
\]
By the triangle inequality for the sub-Gaussian Orlicz norm $\|\cdot\|_{\psi_2}$, we have:
\[
\|Z\|_{\psi_2} = \|\Delta_u - \Delta_v\|_{\psi_2} \leq \|\Delta_u\|_{\psi_2} + \|\Delta_v\|_{\psi_2} \leq  \frac{4\gamma^2 C_W^2}{\sqrt{\ln 2 \cdot d_{qk}} } = 4\sigma.
\]
Let $\tau := \|Z\|_{\psi_2}$.
For any zero-mean sub-Gaussian variable $Z$ with $\|Z\|_{\psi_2} = \tau$, the definition of the sub-Gaussian norm gives $\mathbb{E}\left[\exp\left(Z^2/\tau^2\right)\right] \leq 2$, and hence, by Markov's inequality,
\[
\Prob(|Z| > x) \leq 2 \exp\left( -\frac{x^2}{\tau^2} \right).
\]
For a given $\delta \in (0,1)$, choosing $x = \tau \sqrt{\ln(2/\delta)}$ makes the right-hand side equal to $\delta$.
Substituting $\tau \leq 4\sigma$ and rearranging, we obtain the high-probability bound for $D_{u,v}$:
\[
\Prob\left( \left| D_{u, v} - \mu_{u, v} \right| \leq 4 \sigma \sqrt{ \ln \frac{2}{\delta} } \right) \geq 1 - \delta,
\]
Combining this with the upper bound for $\mu_{u,v}$ completes the proof.
\end{proof}
\subsection{Proof of Theorem~\ref{thm:sparse_decay}}\label{app:proof_sparse_decay}

Using the order-statistics bounds, we establish high-probability guarantees for the truncation mass, which directly reflects the approximation error via Theorem~\ref{thm:js_divergence}.

We first define the partition functions for in-cluster and out-of-cluster attention terms:
\[
Z_{\text{in}} = \sum_{j \in \mathcal{S}_i} \exp(S_{ij}), \quad Z_{\text{out}} = \sum_{j \notin \mathcal{S}_i} \exp(S_{ij})
\]
By definition, the truncation mass can be rewritten as:
\[
\alpha_i = \frac{Z_{\text{out}}}{Z_{\text{in}} + Z_{\text{out}}} = \frac{1}{\frac{Z_{\text{in}}}{Z_{\text{out}}} + 1}
\]
This identity shows that bounding the truncation mass is equivalent to bounding the ratio $R = \frac{Z_{\text{in}}}{Z_{\text{out}}}$: since $\alpha_i = 1/(R+1)$ is strictly decreasing in $R$, an upper bound on $R$ yields a lower bound on $\alpha_i$.
We now return to the proof of Theorem~\ref{thm:sparse_decay}.
\begin{proof}
Note that $S_i$ is the descending rearrangement of $S_{ij}$, we have:
\[
Z_{\text{in}} = \sum_{j \in \mathcal{S}_i} \exp(S_{ij}) \leq M \exp(S_{(1)}), \quad Z_{\text{out}} = \sum_{j \notin \mathcal{S}_i} \exp(S_{ij}) \geq (L-M) \exp(S_{(L)})
\]
This gives the deterministic upper bound:
\[
\frac{Z_{\text{in}}}{Z_{\text{out}}} \leq \frac{M}{L-M} \exp\left( S_{(1)} - S_{(L)} \right)
\]
By Theorem~\ref{thm:order_stat_difference}, for any $\delta \in (0,1)$,
\[
\Prob\left( S_{(1)} - S_{(L)} \leq 2\sqrt{2}\,\sigma\sqrt{\ln(eL)}+\sqrt{2}\sigma+4 \sigma \sqrt{ \ln \frac{2}{\delta} } \right) \geq 1 - \delta
\]
Combining these two results gives the high-probability bound:
\[
\Prob\left( \frac{Z_{\text{in}}}{Z_{\text{out}}} \leq \frac{M}{L-M} \exp\left( 2\sqrt{2}\,\sigma\sqrt{\ln(eL)}+\sqrt{2}\sigma+4 \sigma \sqrt{ \ln \frac{2}{\delta} } \right) \right) \geq 1 - \delta
\]

For the size $M = \mathcal{O}(\sqrt{L})$:
\[
\frac{M}{L-M} = \mathcal{O}\left( \frac{\sqrt{L}}{L} \right) = \mathcal{O}\left( \frac{1}{\sqrt{L}} \right)
\]
The term $\exp(C\sigma\sqrt{\ln L})$ grows sub-polynomially in $L$, so for sufficiently large $L$, the $1/\sqrt{L}$ decay dominates, and the upper bound we get is strictly decreasing in $L$. Since $\alpha_i = \frac{1}{R+1}$ is a strictly decreasing function of $R$, this implies that the high-probability lower bound on $\alpha_i$ is strictly increasing as $L$ grows, that is, the information loss of sparse attention grows with the sequence length, completing the proof.

\end{proof}

\subsection{Additional Concentration Illustration}
\label{app:concentration}
Figure~\ref{fig:r_trend} visualizes the high-probability upper bound on $R$ for the two budget regimes discussed in Section~\ref{sec:core}; it is retained here to keep the main text within the target length.

The interpretation of $M$ as a dilution degree is made precise by the same bound. In practical model design, $M$ is selected to avoid excessive information loss, and the commonly used orders are $\mathcal{O}(L^a)$ with $a<1$ and $\mathcal{O}(\ln L)$; under both cases the high-probability upper bound on $R$ decreases as $L$ increases, so the implied lower bound on the truncation mass $\alpha$ increases with $L$. The derivation uses only $|\mathcal{S}_i|=M$, not which indices are retained, so the statement applies to any $M$-sparse support, and the top-$M$ set, which maximizes the retained mass, is the most favorable case among them.

For the numerical illustration we substitute the representative magnitudes $\gamma = 1$ and $C_W = 2$. As reported in Appendix~\ref{app:norm_stats}, the spectral norms of the projection matrices roughly range from $1.5$ to $2.8$ (median $\approx 2$), and the $L_2$ norms of word embeddings lie in the overall range $0$--$2.5$, with a core concentration between $0.5$ and $2.0$ (median $\approx 1$); since the coefficient scales as $C \propto (\gamma C_W)^2$, substituting the extreme ends of these ranges makes the resulting bound vacuous. With $d_{qk} = 64$ the coefficient is $C \approx 1.7$ and the upper bound decays slowly, whereas with $d_{qk} = 128$ the coefficient is $C \approx 1.2$ and, when $M = \mathcal{O}(\ln L)$, the upper bound decays rapidly, indicating a fast information-decay rate in long-text scenarios.

\begin{figure}[!ht]
  \centering
  \caption{Upper bound of $R$ versus the sequence length $L$}
  \label{fig:r_trend}
    \subfigure{%
      \label{fig:root_n_trend}%
      \includegraphics[width=0.48\linewidth]{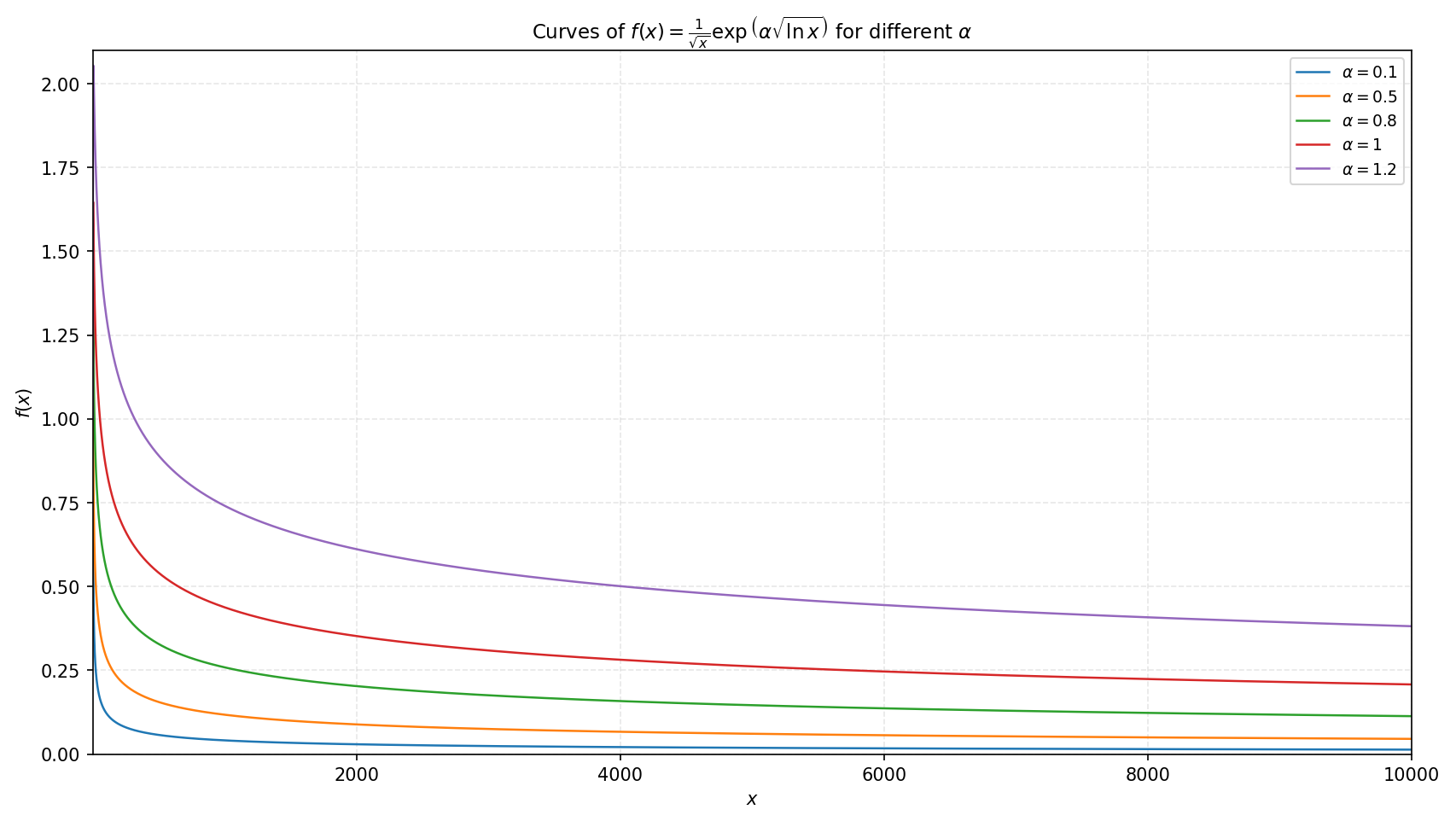}%
    }%
    \hfill
    \subfigure{%
      \label{fig:ln_n_trend}%
      \includegraphics[width=0.48\linewidth]{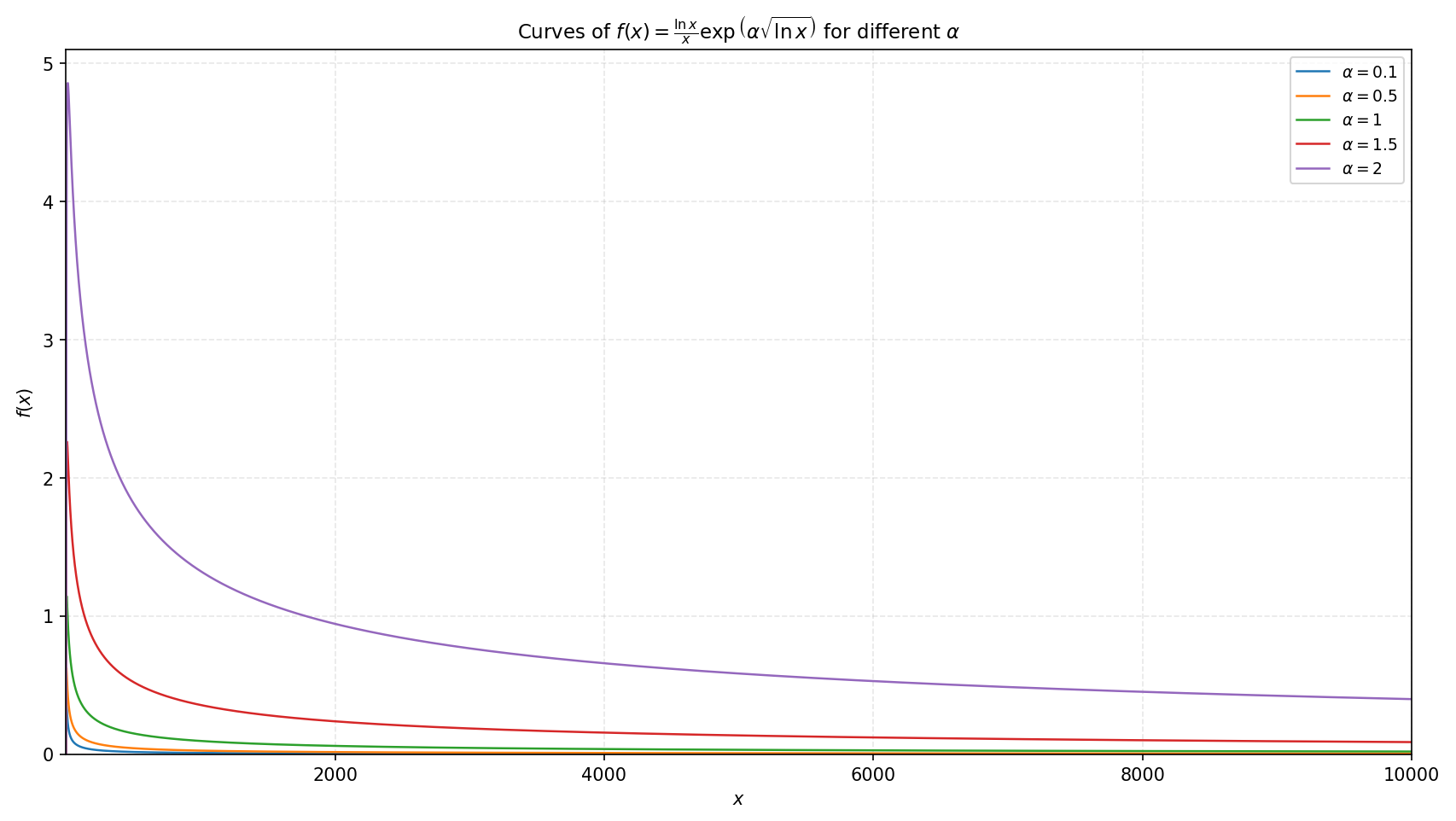}%
    }%
\end{figure}

\subsection{Information Loss Across Sparsity Strategies}
\label{app:strategy-loss}
We compare the truncation mass $\alpha$ and the induced loss $f(\alpha)$ as a function of $L$ for six sparsity strategies, each with its own budget $M(L)$: top-$M$ with $M=\lceil\sqrt{L}\rceil$ and $M=\lceil\ln L\rceil$ (the regimes assumed in Section~\ref{sec:core}), a sliding window of size $32$ (Longformer-style), a window plus random tokens (BigBird-style), a dilated pattern (Sparse Transformer-style), and $M=\lceil\sqrt{L}\rceil$ uniformly random tokens. Figure~\ref{fig:js_vs_l} reports the results for $L\in\{64,128,256,512,1024\}$.

Three observations follow. First, the strategy dominates the sequence length: at $L=1024$ and the same budget $M=32$, top-$M$ truncation discards $\alpha=0.20$ of the attention mass ($D_{\mathrm{JS}}=0.086$), whereas a fixed window discards $\alpha=0.63$ ($0.345$) and a random pattern discards $\alpha=0.97$ ($0.647$). This is expected, since among all supports of size $M$ the top-$M$ set retains the largest mass; it is the most favorable case for the bound of Theorem~\ref{thm:sparse_decay}, which only uses $|\mathcal{S}_i|=M$ and therefore applies to all of the strategies above. Second, for every content-agnostic strategy the information loss grows with $L$: between $L=64$ and $L=1024$ it rises from $0.291$ to $0.345$ (window), from $0.210$ to $0.358$ (dilated) and from $0.548$ to $0.647$ (random), the last approaching $\ln 2$, its maximal value. Third, all six strategies share the same upward trend, and even the slowest one---top-$M$ with $M=\lceil\sqrt{L}\rceil$, whose budget grows fast enough to offset most of the loss---still increases, from $\alpha=0.166$ and $D_{\mathrm{JS}}=0.069$ at $L=64$ to $\alpha=0.202$ and $D_{\mathrm{JS}}=0.086$ at $L=1024$. The direction therefore agrees with Theorem~\ref{thm:sparse_decay}, whose guaranteed lower bound on $\alpha$ also increases with $L$. The two, however, differ by one to three orders of magnitude (dashed line in Figure~\ref{fig:js_vs_l}), because the i.i.d.\ sub-Gaussian model underlying the bound does not capture how strongly real attention concentrates on a few keys: the bound is conservative rather than tight, and it is this lack of tightness---not the direction of the effect---that separates the theory from the measurements.

\begin{figure}[!ht]
\centering
\includegraphics[width=0.8\linewidth]{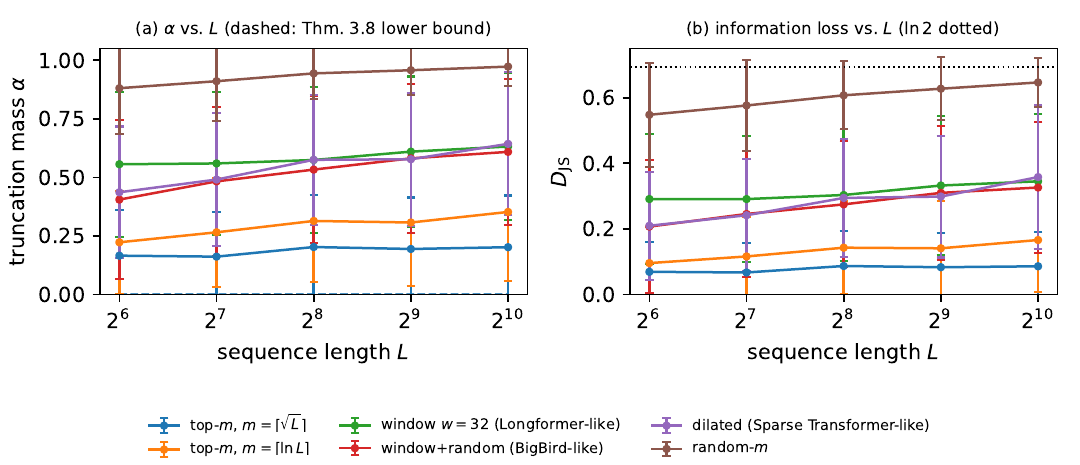}
\caption{Truncation mass $\alpha$ (a) and information loss $D_{\mathrm{JS}}=f(\alpha)$ (b) versus the sequence length $L$, for six sparsity strategies with their own budget $M(L)$. Error bars are $\pm1$ standard deviation over $12\times12\times8$ attention rows. The dashed line in (a) is the lower bound on $\alpha$ implied by Theorem~\ref{thm:sparse_decay}.}
\label{fig:js_vs_l}
\end{figure}

\section{Detailed Proof of Rademacher Complexity} \label{app:C}
This appendix proves the Rademacher-complexity results of Section~\ref{sec:generalization}: the covering-number bound for the sparse class, the corresponding generalization bound, the Lipschitz property of a fixed-support attention layer, and its depth-$H$ cascade.

\subsection{Proof of Theorem~\ref{thm:sparse-linear-rademacher}}
\label{app:C1}

Fix an output coordinate $p\in[d]$ and a query position
$t\in[L]$. Consider the position-wise static sparse linear class
\[
\mathcal{H}_{p,t}
=
\left\{
h_t(x_{1:L})
=
\sum_{s\in I}w_s[x_s]_p
\;\middle|\;
I\subseteq[L],\ |I|\leq M,\ \|w\|_1\leq1
\right\},
\]
where the support $I$ and coefficients $w$ are independent of the
input. Assume that $\|x_s\|_2\leq\gamma$ for every input sequence
and every position $s\in[L]$.

For fixed convex pooling weights
$\lambda\in\Delta^{L-1}$, define
\[
\mathcal{H}_p
=
\left\{
x_{1:L}\mapsto
\sum_{t=1}^{L}\lambda_t h_t(x_{1:L})
\;\middle|\;
h_t\in\mathcal{H}_{p,t}
\text{ for every }t\in[L]
\right\}.
\]

We first establish a sample-wise bound for the pooled class.
Fix a sample $S=(x^{(1)},\ldots,x^{(n)})$ and let
$\boldsymbol{\sigma}=(\sigma_1,\ldots,\sigma_n)$ be independent
Rademacher random variables. By subadditivity of the supremum
and the nonnegativity of the pooling weights,
\begin{align*}
\widehat{\mathcal{R}}_S(\mathcal{H}_p)
&=
\mathbb{E}_{\boldsymbol{\sigma}}
\left[
\sup_{h\in\mathcal{H}_p}
\frac{1}{n}\sum_{i=1}^{n}
\sigma_i
\sum_{t=1}^{L}\lambda_t h_t(x^{(i)})
\right]\\
&\leq
\sum_{t=1}^{L}\lambda_t
\mathbb{E}_{\boldsymbol{\sigma}}
\left[
\sup_{h_t\in\mathcal{H}_{p,t}}
\frac{1}{n}\sum_{i=1}^{n}
\sigma_i h_t(x^{(i)})
\right]\\
&=
\sum_{t=1}^{L}\lambda_t
\widehat{\mathcal{R}}_S(\mathcal{H}_{p,t})\\
&\leq
\max_{t\in[L]}
\widehat{\mathcal{R}}_S(\mathcal{H}_{p,t}).
\end{align*}
It therefore suffices to derive an upper bound on the
position-wise empirical Rademacher complexities that is uniform
over both the sample $S$ and the position $t$.

For every $h_t\in\mathcal{H}_{p,t}$, the coefficient constraint
and the input bound imply
\[
|h_t(x_{1:L})|
\leq
\sum_{s\in I}|w_s|\,|[x_s]_p|
\leq
\gamma\|w\|_1
\leq\gamma.
\]
We now bound the empirical Rademacher complexity of
$\mathcal{H}_{p,t}$ using Dudley's entropy integral
\citep{dudley1967sizes,vershynin2020high}. For a fixed sample
$S$, define the induced supremum norm by
\[
\|h-g\|_{S,\infty}
=
\max_{1\leq i\leq n}
|h(x^{(i)})-g(x^{(i)})|.
\]
Applying Dudley's entropy-integral bound, with universal
constants absorbed into $C_0$, gives
\[
\widehat{\mathcal{R}}_S(\mathcal{H}_{p,t})
\leq
\frac{C_0}{\sqrt{n}}
\int_0^{2\gamma}
\sqrt{
\log N\bigl(
\epsilon,\mathcal{H}_{p,t},\|\cdot\|_{S,\infty}
\bigr)
}\,d\epsilon,
\]
where $C_0$ is an absolute constant and $N(\epsilon,\cdot,\cdot)$
denotes the covering number.

To estimate the covering number, decompose the class over
candidate supports:
\[
\mathcal{H}_{p,t}
=
\bigcup_{\substack{I\subseteq[L]\\|I|\leq M}}
\mathcal{H}_{p,t,I},
\]
where
\[
\mathcal{H}_{p,t,I}
=
\left\{
x_{1:L}\mapsto
\sum_{s\in I}w_s[x_s]_p
\;\middle|\;
\|w\|_1\leq1
\right\}.
\]
Any support of size smaller than $M$ can be extended to a support
of size $M$ by adding zero coefficients. Thus, for $1\leq M\leq L$,
the class can equivalently be represented as a union over supports
of size exactly $M$, whose number is $\binom{L}{M}$.

For a fixed support $I$ of size $M$, let $w,w'\in B_1^M$, where
$B_1^M=\{u\in\mathbb{R}^M:\|u\|_1\leq1\}$. Then
\[
\begin{aligned}
\|h_w-h_{w'}\|_{S,\infty}
&=
\max_{1\leq i\leq n}
\left|
\sum_{s\in I}(w_s-w'_s)[x_s^{(i)}]_p
\right|\\
&\leq
\gamma\|w-w'\|_1.
\end{aligned}
\]
Hence an $(\epsilon/\gamma)$-cover of $B_1^M$ in the $\ell_1$
norm induces an $\epsilon$-cover of $\mathcal{H}_{p,t,I}$ in the
sample supremum norm. By the standard volumetric estimate,
\[
N\bigl(
\epsilon,\mathcal{H}_{p,t,I},\|\cdot\|_{S,\infty}
\bigr)
\leq
\left(1+\frac{2\gamma}{\epsilon}\right)^M.
\]
Taking the union over candidate supports and using
\[
\binom{L}{M}\leq\left(\frac{eL}{M}\right)^M,
\]
we obtain
\[
N\bigl(
\epsilon,\mathcal{H}_{p,t},\|\cdot\|_{S,\infty}
\bigr)
\leq
\left(\frac{eL}{M}\right)^M
\left(1+\frac{2\gamma}{\epsilon}\right)^M.
\]
For $0<\epsilon\leq\gamma$, this implies
\[
N\bigl(
\epsilon,\mathcal{H}_{p,t},\|\cdot\|_{S,\infty}
\bigr)
\leq
\left(\frac{3eL\gamma}{M\epsilon}\right)^M.
\]
For $\gamma<\epsilon\leq2\gamma$, we instead use
$1+2\gamma/\epsilon\leq3$ to obtain
\[
\log N\bigl(
\epsilon,\mathcal{H}_{p,t},\|\cdot\|_{S,\infty}
\bigr)
\leq
M\log\frac{3eL}{M}.
\]

Substituting these estimates into Dudley's integral and splitting
the integral at $\gamma$, we obtain
\begin{align*}
\widehat{\mathcal{R}}_S(\mathcal{H}_{p,t})
&\leq
\frac{C_0}{\sqrt{n}}
\left[
\int_0^\gamma
\sqrt{
M\log\frac{3eL\gamma}{M\epsilon}
}\,d\epsilon
+
\int_\gamma^{2\gamma}
\sqrt{
M\log\frac{3eL}{M}
}\,d\epsilon
\right]\\
&=
C_0\gamma\sqrt{\frac{M}{n}}
\left[
\int_0^1
\sqrt{
\log\frac{3eL}{M}+\log\frac{1}{u}
}\,du
+
\sqrt{\log\frac{3eL}{M}}
\right],
\end{align*}
where $u=\epsilon/\gamma$ in the first integral.

For $a\geq0$, the inequality
\[
\int_0^1\sqrt{a+\log(1/u)}\,du
\leq
\sqrt{a}+\frac{\sqrt{\pi}}{2}
\]
therefore gives
\[
\widehat{\mathcal{R}}_S(\mathcal{H}_{p,t})
\leq
C\gamma\sqrt{\frac{M}{n}}
\left(
\sqrt{\log\frac{3eL}{M}}
+
\frac{\sqrt{\pi}}{2}
\right),
\]
where $C$ absorbs the absolute constants. This estimate is
uniform over the sample $S$, the position $t$, and the output
coordinate $p$.

Combining this uniform position-wise estimate with the
sample-wise convex-pooling inequality established above, we have
\[
\widehat{\mathcal{R}}_S(\mathcal{H}_p)
\leq
C\gamma\sqrt{\frac{M}{n}}
\left(
\sqrt{\log\frac{3eL}{M}}
+
\frac{\sqrt{\pi}}{2}
\right)
\]
for every sample $S$. Taking expectation over $S\sim D^n$
yields
\[
\mathcal{R}_n(\mathcal{H}_p)
\leq
C\gamma\sqrt{\frac{M}{n}}
\left(
\sqrt{\log\frac{3eL}{M}}
+
\frac{\sqrt{\pi}}{2}
\right).
\]
Since this bound is independent of the output coordinate $p$,
it holds uniformly for every $p\in[d]$. This completes the proof
of Theorem~\ref{thm:sparse-linear-rademacher}.

\subsection{Proof of Theorem~\ref{thm:generalization-sparse-linear}}
\label{app:C2}

By Theorem~\ref{thm:sparse-linear-rademacher}, the expected
Rademacher complexity of the static sparse aggregation class
$\mathcal{H}$ satisfies
\[
\mathcal{R}_n(\mathcal{H})
\leq
C\gamma\sqrt{\frac{M}{n}}
\left(
\sqrt{\log\frac{3eL}{M}}
+
\frac{\sqrt{\pi}}{2}
\right).
\]

For each fixed sample
$S=((X_1,Y_1),\ldots,(X_n,Y_n))$, define the associated loss
class
\[
\mathcal{L}_{\mathcal{H}}
=
\left\{
(X,Y)\mapsto\ell(h(X),Y)
\;\middle|\;
h\in\mathcal{H}
\right\}.
\]
Assume that $\mathcal{H}$ is a scalar-valued function class and
that, for every $y\in\mathcal{Y}$, the map
$z\mapsto\ell(z,y)$ is $1$-Lipschitz. By the
Ledoux--Talagrand contraction inequality, the empirical
Rademacher complexity of the loss class satisfies
\[
\widehat{\mathcal{R}}_S(\mathcal{L}_{\mathcal{H}})
\leq
c_{\mathrm{ctr}}\,
\widehat{\mathcal{R}}_{S_X}(\mathcal{H}),
\]
where $S_X=(X_1,\ldots,X_n)$ and
$c_{\mathrm{ctr}}$ is an absolute constant determined by the
chosen convention for Rademacher complexity. Taking expectation
over the sample gives
\[
\mathcal{R}_n(\mathcal{L}_{\mathcal{H}})
\leq
c_{\mathrm{ctr}}\,\mathcal{R}_n(\mathcal{H}).
\]
The absolute constant $c_{\mathrm{ctr}}$ will be absorbed into
the constant in the final bound.

Because the loss takes values in $[0,1]$, the standard uniform
generalization inequality
\citep{bartlett2002rademacher} implies that, with probability
at least $1-\delta$, every $h\in\mathcal{H}$ satisfies
\[
\begin{aligned}
\mathbb{E}[\ell(h(X),Y)]
\leq{}&
\frac{1}{n}\sum_{i=1}^{n}\ell(h(X_i),Y_i)
+
2\mathcal{R}_n(\mathcal{L}_{\mathcal{H}})
+
\sqrt{\frac{8\ln(2/\delta)}{n}}.
\end{aligned}
\]
Substituting the Rademacher complexity bound from
Theorem~\ref{thm:sparse-linear-rademacher} gives
\[
\begin{aligned}
\mathbb{E}[\ell(h(X),Y)]
\leq{}&
\frac{1}{n}\sum_{i=1}^{n}\ell(h(X_i),Y_i)\\
&+
C'\gamma\sqrt{\frac{M}{n}}
\left(
\sqrt{\log\frac{3eL}{M}}
+
\frac{\sqrt{\pi}}{2}
\right)
+
\sqrt{\frac{8\ln(2/\delta)}{n}},
\end{aligned}
\]
where $C'$ absorbs the constants from the contraction and
uniform generalization inequalities. Renaming $C'$ as $C$
completes the proof of Theorem~\ref{thm:generalization-sparse-linear}.

\subsection{Proof of Lemma~\ref{lem:lipschitz-attention}}\label{app:C3}
\begin{proof}

Consider a sparse attention layer $\mathbf{f}(x)$ with parameters $(W_Q, W_K, W_V)$. We assume the input $x_t \in \mathbb{R}^d$ is bounded such that $\|x_t\|_2 \le \gamma$. Let the weight matrices have spectral norms bounded by $C_W$

For any two input sequences $x, y$, we fix a position $t$ and decompose the difference into value change and weight change:
\begin{align*}
|\mathbf{f}(x)_t - \mathbf{f}(y)_t| &\le \underbrace{\left| \sum_{s \in \mathcal{S}} a_{t,s}(x) (x_s W^V - y_s W^V) \right|}_{\text{Term A: change in values}} + \underbrace{\left| \sum_{s \in \mathcal{S}} (a_{t,s}(x) - a_{t,s}(y)) y_s W^V \right|}_{\text{Term B: change in attention weights}}
\end{align*}
where $\mathcal{S}$ is the sparsity support set ($|\mathcal{S}|=M$).

\paragraph{Bounding Term A:}
Since $\sum_s a_{t,s} = 1$, by the property of spectral norms:
\begin{align*}
\text{Term A} &= \left\| \sum_{s \in \mathcal{S}} a_{t,s}(x) (x_s W^V - y_s W^V) \right\|_2 \\
&\le \sum_{s \in \mathcal{S}} a_{t,s}(x) \| (x_s - y_s) W^V \|_2 \quad \text{(Triangle inequality and $a \ge 0$) } \\
&\le \sum_{s \in \mathcal{S}} a_{t,s}(x) \left( \|W^V\|_2 \cdot \|x_s - y_s\|_2 \right) \quad \text{(property of spectral norms: $\|Av\| \le \|A\| \cdot \|v\|$)} \\
&\le \left( \sum_{s \in \mathcal{S}} a_{t,s}(x) \right) \cdot \left( C_W \cdot \max_{s \in \mathcal{S}} \|x_s - y_s\|_2 \right) \\
&= 1 \cdot C_W \cdot \|x - y\|_\infty.
\end{align*}

\paragraph{Bounding Term B:}
By H\"{o}lder's inequality \citep[Theorem 1.1]{stein2011functional}, \begin{align*}
\left\| \sum_{s \in \mathcal{S}} \Delta a_{t,s} (y_s W^V) \right\|_2 & \leq \sum_{s \in \mathcal{S}} |\Delta a_{t,s}| \cdot \|y_s W^V\|_2 \quad \text{(triangle inequality)} \\
& \leq \left( \sum_{s \in \mathcal{S}} |\Delta a_{t,s}| \cdot 1 \right) \max_{s \in \mathcal{S}} \|y_s W^V\|_2 \quad \text{(H\"{o}lder)} \\
& = \|a_t(x) - a_t(y)\|_1 \cdot \max_{s \in \mathcal{S}} \|y_s W^V\|_2 \\
\end{align*}

From the bounded input assumption, $\max_s \|y_s W^V\|_2 \le \gamma C_W$

\paragraph{Softmax Lipschitz Bound.}
From \citet{nair2026softmax}, softmax is $1/2$-Lipschitz under any $\ell_p$ norm:
\[
\|a_t(x) - a_t(y)\|_1 \le \frac{1}{2} \|z(x) - z(y)\|_1.
\]
Using norm equivalence:
\[
\|z(x) - z(y)\|_1 \le M \|z(x) - z(y)\|_\infty.
\]

\paragraph{Bounding Logits Difference.}
\begin{align*}
|z_s(x) - z_s(y)|
&= \frac{1}{\sqrt{d}} \left| x_t W^Q W^{K\top} x_s^\top - y_t W^Q W^{K\top} y_s^\top \right| \\
&\le \frac{1}{\sqrt{d}} \Big(
|x_t W^Q W^{K\top} (x_s - y_s)^\top|
+ |(x_t - y_t) W^Q W^{K\top} y_s^\top|
\Big) \\
&\le \frac{1}{\sqrt{d}} \left(
\|x_t\|_2 C_W^2 \|x_s - y_s\|_2
+ \|x_t - y_t\|_2 C_W^2 \|y_s\|_2
\right) \\
&\le \frac{2 \gamma C_W^2}{\sqrt{d}} \|x - y\|_\infty.
\end{align*}

Thus,
\[
\|z(x) - z(y)\|_\infty \le \frac{2 \gamma C_W^2}{\sqrt{d}} \|x - y\|_\infty.
\]

\paragraph{Combining Bounds.}
\begin{align*}
\|a_t(x) - a_t(y)\|_1
&\le \frac{1}{2} \cdot M \cdot \frac{2 \gamma C_W^2}{\sqrt{d}} \|x - y\|_\infty \\
&= \frac{\gamma C_W^2 M}{\sqrt{d}} \|x - y\|_\infty.
\end{align*}

Thus,
\[
\text{Term B} \le \frac{\gamma^2 C_W^3 M}{\sqrt{d}} \|x - y\|_\infty.
\]

\paragraph{Final Bound.}
\begin{align*}
\|\mathbf{f}(x)_t - \mathbf{f}(y)_t\|_2
&\le \left( C_W + \frac{\gamma^2 C_W^3 M}{\sqrt{d}} \right) \|x - y\|_\infty.
\end{align*}

Taking maximum over $t$ gives:
\[
\|\mathbf{f}(x) - \mathbf{f}(y)\|_\infty
\le \left( C_W + \frac{\gamma^2 C_W^3 M}{\sqrt{d}} \right) \|x - y\|_\infty.
\]

\end{proof}

\subsection{Proof of Lemma 4.4}\label{app:proof-lemma-16}
\begin{proof}
Let $\mathcal{H}^{(H)} = \mathcal{H} \circ f_{H-1} \circ \cdots \circ f_1$ be the composition of $H-1$ sparse attention layers (each $L_{\text{att}}$-Lipschitz by Lemma~15) with the static sparse aggregation class $\mathcal{H}$. By the Talagrand Contraction Lemma for Lipschitz functions, recursive application of the contraction principle across the $H-1$ layers yields
\begin{align*}
    \mathcal{R}_n(\mathcal{H}^{(H)}) &\le L_{\text{att}} \cdot \mathcal{R}_n(\mathcal{H}^{(H-1)}) \\
    &\le (L_{\text{att}})^{H-1} \cdot \mathcal{R}_n(\mathcal{H}).
\end{align*}
Substituting the single-layer Rademacher bound for $\mathcal{H}$ from Theorem~13 gives
\[
    \mathcal{R}_n(\mathcal{H}^{(H)}) \le (L_{\text{att}})^{H-1} \cdot C\gamma\sqrt{\frac{M}{n}}\left(\sqrt{\log\frac{3eL}{M}} + \frac{\sqrt{\pi}}{2}\right).
\]
\end{proof}

\subsection{Rademacher Complexity Generalization Bound}
\label{sec:exp-rade}

We evaluate the generalization stability of the sparse attention mechanism by measuring the gap between training and test MSE across varying sparsity degrees ($M$) and number of layers $H$.

Following the methodology of randomization tests proposed by \citet{zhang2017understanding}, we employ completely unstructured Gaussian noise to empirically estimate the Rademacher complexity of the M-sparse attention mechanism. Since Rademacher complexity is defined by a model's ability to fit random noise, this approach provides an empirical proxy for the effective capacity of the considered sparse-attention function class, while reducing the influence of data-specific patterns. By observing how the generalization gap scales with sparsity $M$ and depth $H$ under pure noise, we can isolate the structural constraints imposed by sparse attention on the model's propensity for 'brute-force memorization', providing a direct experimental proxy for the theoretical bounds derived in Theorem~\ref{thm:generalization-sparse-linear}.

Increasing $M$ broadens the empirical generalization gap, consistent with the predicted growth of the Rademacher complexity; Figure~\ref{fig:rade_m} compares the theoretical and empirical trends.
\begin{figure}[!ht]
\centering
\subfigure[Theoretical bound versus $M$]{%
  \includegraphics[width=0.45\linewidth]{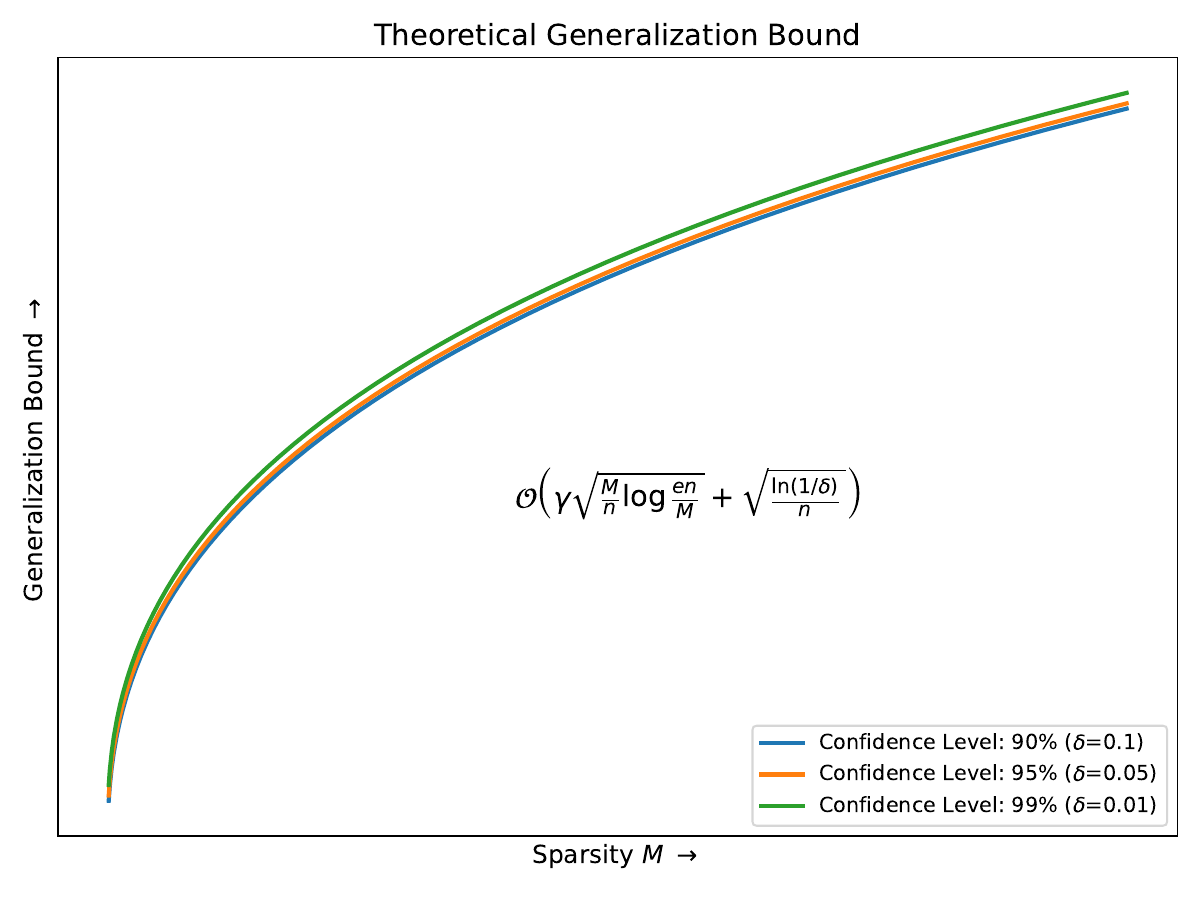}%
  \label{fig:theory_m}%
}
\subfigure[Empirical gap versus $M$]{%
  \includegraphics[width=0.45\linewidth]{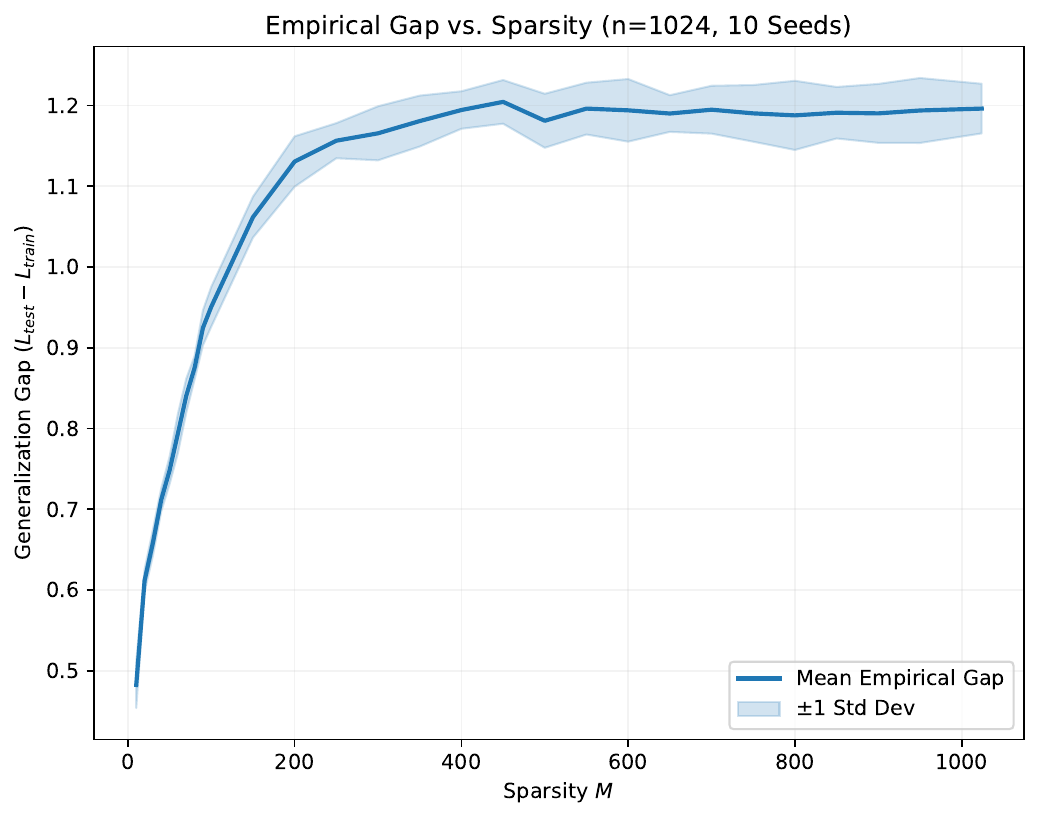}%
  \label{fig:exp_m}%
}
\caption{Generalization bound versus the sparsity degree $M$.}
\label{fig:rade_m}
\end{figure}

\section{Proof of Theorem~\ref{thm:main_gen_bound}}
\label{app:proof-entropy-bound}

We prove the bound in three steps.

\paragraph{Step 1: Xu--Raginsky Bound.}
By Theorem~1 of~\citet{xu2017information}, if
$\ell(w,Z)$ is $\sigma_\ell$-sub-Gaussian under $\mu$ for all
$w\in\mathcal{W}$, then
\[
\bigl|\mathrm{gen}(\mu,P_{W|S})\bigr|
\;\le\;
\sqrt{\frac{2\sigma_\ell^2}{n}\,I(S;W)} .
\]
Since $\ell\in[0,1]$, Hoeffding's lemma implies that $\ell(w,Z)$ is
$1/2$-sub-Gaussian, that is, $\sigma_\ell^2=1/4$. Hence
\begin{equation}
\label{eq:proof-xu}
\bigl|\mathrm{gen}(\mu,P_{W|S})\bigr|
\;\le\;
\sqrt{\frac{I(S;W)}{2n}} .
\end{equation}

\paragraph{Step 2: Entropy Upper Bound on $I(S;W)$.}
For any pair of random variables, $I(S;W)\le H(W)$, since
$I(S;W)=H(W)-H(W\mid S)$. With $W=(I,q)$ and the chain rule of entropy,
\[
H(W)=H(I)+H(q\mid I).
\]

\emph{Bounding $H(I)$.} The support set $I$ is a subset of
$\{1,\dots,L\}$ of cardinality $M$. The number of such subsets is
$\binom{L}{M}$, so
\[
H(I)\;\le\;\log\binom{L}{M}.
\]
Using the standard combinatorial bound
$\binom{L}{M}\le (eL/M)^M$, we obtain
\begin{equation}
\label{eq:proof-HI}
H(I)\;\le\;M\log\frac{eL}{M}.
\end{equation}

\emph{Bounding $H(q\mid I)$.} Conditioned on $I$, the weight vector $q$
lies in the $M$-dimensional probability simplex
$\Delta^{M-1}$. We assume that $q$ is restricted to an $\epsilon$-net of
$\Delta^{M-1}$ under the $\ell_\infty$ norm. The covering number of the
simplex satisfies
\[
N\!\left(\epsilon,\Delta^{M-1},\|\cdot\|_\infty\right)
\;\le\;
\left(1+\frac{2}{\epsilon}\right)^{M}.
\]
Since the entropy of a random variable supported on a finite set of size
$N$ is at most $\log N$,
\begin{equation}
\label{eq:proof-Hq}
H(q\mid I)
\;\le\;
\log N\!\left(\epsilon,\Delta^{M-1},\|\cdot\|_\infty\right)
\;\le\;
M\log\!\left(1+\frac{2}{\epsilon}\right).
\end{equation}

Combining \eqref{eq:proof-HI} and \eqref{eq:proof-Hq},
\begin{equation}
\label{eq:proof-ISW}
I(S;W)
\;\le\;
H(W)
\;\le\;
M\log\frac{eL}{M}
\;+\;
M\log\!\left(1+\frac{2}{\epsilon}\right).
\end{equation}

\paragraph{Step 3: Final Bound.}
Substituting \eqref{eq:proof-ISW} into \eqref{eq:proof-xu} yields
\[
\bigl|\mathrm{gen}(\mu,P_{W|S})\bigr|
\;\le\;
\sqrt{
\frac{M}{2n}
\left(
\log\frac{eL}{M}
+
\log\!\left(1+\frac{2}{\epsilon}\right)
\right)
},
\]
which completes the proof.

\end{document}